\documentclass[article]{siamonline190516}

\usepackage{amsfonts,amssymb}
\usepackage{graphicx}
\usepackage{multirow}
\usepackage{subcaption}
\usepackage{mathtools}
\usepackage{psfrag}
\usepackage{dirtytalk}
\usepackage{bbm}
\usepackage{epstopdf}
\usepackage{algorithm}
\usepackage{algorithmic}
\usepackage[permil]{overpic}

\usepackage[shortlabels]{enumitem}
\usepackage{xspace}

\ifpdf
  \DeclareGraphicsExtensions{.eps,.pdf,.png,.jpg,.jpeg}
\else
  \DeclareGraphicsExtensions{.eps}
  \fi

\graphicspath{{./figures/}}

\usepackage{enumitem}
\setlist[enumerate]{leftmargin=.5in}
\setlist[itemize]{leftmargin=.5in}

\newsiamthm{claim}{Claim}
\newsiamremark{conjecture}{Conjecture}
\newsiamremark{remark}{Remark}
\newsiamremark{example}{Example}
\newsiamremark{hypothesis}{Hypothesis}
\newsiamremark{problem}{Problem}
\newsiamremark{assumption}{Assumption}

\usepackage{etoolbox}
\AtEndEnvironment{remark}{\null\hfill$\Diamond$}

\usepackage{pict2e}

\newcommand{\mcl}{\mathcal}

\newcommand{\mbb}{\mathbb}

\newcommand{\Diam}{{\rm Diam}}

\newcommand{\dd}{{\rm d}}

\newcommand{\RR}{\mathbb R}

\newcommand{\Y}{\mcl{Y}}

\newcommand{\U}{\mcl{U}}

\newcommand{\V}{{\mcl{V}}}
\newcommand{\W}{{\mcl{W}}}

\newcommand{\Cstab}{C_{{\rm stab}}}

\newcommand{\Lip}{{\rm Lip}}

\newcommand{\T}{\mathcal{T}}

\newcommand{\F}{\mathcal{F}}

\newcommand{\D}{\mathcal{D}}

\newcommand{\PP}{\mbb{P}}

\newcommand{\R}{\mbb{R}}
\newcommand{\EE}{\mbb{E}}

\definecolor{darkred}{rgb}{.7,0,0}

\definecolor{darkgreen}{rgb}{.15,.55,0}

\definecolor{darkblue}{rgb}{0,0,0.7}

\theoremstyle{plain}
\theoremheaderfont{\normalfont\sffamily}
\theorembodyfont{\normalfont\itshape}
\theoremseparator{.}
\theoremsymbol{}

\usepackage{amsopn}

\DeclareMathOperator*{\argmin}{arg\,min}
\DeclareMathOperator*{\argmax}{arg\,max}

\reversemarginpar
\usepackage[colorinlistoftodos,prependcaption,textsize=tiny]{todonotes}

\headers{Data-Driven BIPs with Generative Priors}{B. Hosseini and 
Z. Huang}

\title{Error Analysis of Bayesian Inverse Problems 
with Generative Priors}

\author{ Bamdad Hosseini\thanks{University of Washington, Seattle, WA 98195, USA
  (\email{bamdadh@uw.edu, ziqih2@uw.edu} \\ 
  This work was  supported by the National Science Foundation grants 
NSF-DMS-2208535 (Machine Learning for Bayesian Inverse Problems) and 
NSF-DMS-2337678 (CAREER: Gaussian Processes for Scientific Machine
Learning: Theoretical Analysis and Computational Algorithms).
  )} 
   \and Ziqi Huang\footnotemark[1]
  }

\ifpdf
\hypersetup{
  pdftitle={Data Driven Bayesian Inverse Problem with Neural Network},
  pdfauthor={Bamdad Hosseini, Ziqi Huang}
}
\fi

\begin{document}
\newcommand{\RefM}{\eta}
\newcommand{\TargetM}{\nu}

\maketitle

\begin{abstract} 
Data-driven methods for the solution of inverse problems have become widely popular in 
recent years thanks to the rise of machine learning techniques. A popular approach concerns 
the training of a generative model on additional data to learn a bespoke prior for the 
problem at hand. In this article we present an analysis for such problems by 
presenting quantitative error bounds for minimum Wasserstein-2 generative models for the prior. 
We show that under some assumptions, the error in the posterior due to the generative prior will inherit the same rate as the prior with respect to the Wasserstein-1 distance. 
We further present numerical experiments that verify that aspects of our error analysis manifest in some benchmarks followed by an elliptic PDE inverse problem where a generative prior 
is used to model a non-stationary field.
\end{abstract}

\begin{keywords} 
Bayesian Inverse problems, Data-driven inference, Generative Models, Wasserstein Metric 
\end{keywords}

\begin{AMS}
62F15, 
62G86, 
49Q22. 
\end{AMS}

\sloppy
\section{Introduction}\label{sec:Introduction}

The Bayesian approach to inverse problems \cite{stuart2010inverse,kaipio2005statistical} is now central to many computational frameworks in scientific computing and 
uncertainty quantification (UQ) \cite{riedmaier2021unified, zhang2020basic, sullivan2015introduction}. In its most 
general form Bayes' rule is expressed in terms of the Radon-Nikodym 
derivative of the posterior measure 
with respect to the prior\footnote{We use the notation 
$\mu(f) = \int_\U f(u) \mu(\dd u)$ for a probability measure $\mu$
and a measurable function $f$.},
\begin{equation}\label{bayes-rule}
    \frac{\dd \nu}{\dd \mu}(u) = \frac{1}{Z(y)}
    \exp(-\Phi(u;y)), \qquad Z(y) = \mu \big(  \exp( - \Phi(u; y) ) 
  \big),  
\end{equation}
where $u$ denotes the {\it unknown parameter} of interest that is 
assumed to be in a Banach  space $\U$ while
$y$ denotes the {\it observations or measurements} for the problem 
which may also belong to a Banach space $\Y$. The two are 
connected via the {\it likelihood potential} $\Phi: \U \times \Y 
\to \RR$ (also known as the negative log-likelihood function)
which encodes the likelihood of observing $y$ given $u$. 
$\mu \in \PP(\U)$ is the {\it prior} probability measure, that 
encodes our prior knowledge about $u$ while $\nu \in \PP(\U)$ 
is the {\it posterior} probability measure that represents an updated prior after observing $y$. Finally, $Z(y)$ is a 
normalizing constant referred to as the {\it evidence}, which 
ensures that $\nu$ is a well-defined probability measure. 

Since the observations $y$ are often limited and noisy, 
the exact recovery of $u$ is impossible which motivates 
the use of Bayes' rule in the first place. Broadly speaking, if the 
observations $y$ were not informative, then the likelihood potential
$\Phi$ would be small, which means that many values of $u$ can explain 
the observations. In this case we expect that the posterior $\nu$ would 
be close to the prior $\mu$. Conversely, when the observations are very informative, we expect $\nu$ to concentrate around the true value of $u$. 
Many real-world applications of interest fall in the weakly or mildly informative observation regime, in which case, 
the prior has a major influence on the shape and support of the posterior 
measure $\nu$. It is precisely in these cases where the choice of the prior 
measure $\mu$ becomes extremely important as it directly influences 
the quality of the posterior as well as our UQ estimates  
in downstream tasks. 

Hence, the design and choice of the prior $\mu$ is of great interest 
in the theory, application, and algorithmic development
of Bayesian inverse problems (BIPs). Traditionally, the priors 
are designed by experts by balancing theoretical and algorithmic 
convenience with the specific needs of the problem at hand; see for example \cite{gelman2013bayesian} or \cite[Sec.~3.3]{kaipio2005statistical} as well as  \Cref{sec:literature-review}.
However, in the past decade, thanks to the rise of 
machine learning (ML) and generative modeling, a new 
class of methodologies have been proposed where the prior 
$\mu$ is learned or designed in an adaptive manner. This 
class largely falls under the category of data-driven 
methods for solving inverse problems in the parlance of
\cite{arridge2019solving, ongie2020deep, mukherjee2023learned}.

We can categorize data-driven inversion methods under 
three broad families: (1) learned prior models,
where an ML model is trained on an additional empirical 
data set to learn an empirical prior $\widehat{\mu}$ in 
the Bayesian context or equivalently a regularizer in the 
deterministic context \cite{bora2017compressed,lunz2018adversarial,mukherjee2024data, bohra2023bayesian,patel2022solution, ongie2020deep};
(2) plug-and-play priors that utilize  
denoising operators to learn algorithmic priors and regularizers 
\cite{venkatakrishnan2013plug,zhang2017learning, zhang2022plug,laumont2022bayesian}; and (3) the more recent class of simulation-based inference methods 
that circumvent the need for explicit knowledge of the likelihood 
potential or the prior measure and only require black-box 
methods that can sample the prior and simulate the observations
\cite{lueckmann2021benchmarking,song2021solving, baptista2024conditional,zammit2025neural}.
Our goal in this article is to develop an error analysis for the 
first class of BIP methods that utilize a learned prior, 
perhaps using a generative model. 

\subsection{Problem statement and summary of contributions}
\label{sec:problem-statement-and-contribution-summary}

We consider an approximation to \eqref{bayes-rule} of the form 
\begin{equation}\label{eq:bayes-rule-approx}
    \frac{\dd \widehat{\nu}}{\dd \widehat{\mu}}(u)
    = \frac{1}{\widehat{Z}(y)} \exp\left( - \Phi(u; y) \right),
    \qquad \widehat{Z}(y) = \widehat{\mu} \big( \exp( - \Phi(u; y) ) \big),
\end{equation}
where $\widehat{\nu}$ is an approximate posterior 
measure that arises due to an approximate prior $\widehat{\mu}$--
note that we have assumed that the likelihood potential $\Phi$ is 
fixed as we are primarily interested in prior approximations, 
although we will discuss likelihood perturbations later. 
We further assume that $\widehat{\mu}$ is defined as the 
pushforward of a reference measure $\eta \in \PP(\U)$, that is,
$\widehat{\mu} = \widehat{T} \# \eta,$
where $\#$ denotes the pushforward operator\footnote{Given a measure $\eta \in \PP(\U)$
and a measurable map $T: \U \to \V$, we define $T \# \eta \in \PP(\V)$
via the identity $T \# \eta (A) = \eta (T^{-1}(A))$ for any 
measurable set $A$ with $T^{-1}(A)$ denoting the pre-image of $A$
under $T$.} and $\widehat{T}: \U \to \U$ is a transport map which represents a generative model. The transport model encompasses many popular generative models including 
Generative Adversarial Nets (GANs) \cite{goodfellow2020generative}, 
Normalizing flows (NFs) \cite{kobyzev2021normalizing}, 
and Flow Matching \cite{lipman2023flow}.

\begin{quote}
    \emph{Our primary goal is to quantify the error between 
    the approximate posterior $\widehat{\nu}$
    and the true posterior $\nu$, due to the limitations 
    in computing $\widehat{T}$ arising from finite training 
    data  and limited  parameterizations.
    }
\end{quote}

We emphasize that we are mainly concerned here with the 
approximation of the posterior $\nu$ due to the approximation of 
the prior $\mu$ with the generative model $\widehat{\mu}$, that is, 
we view $\mu$ as a ground truth prior that the generative model 
targets and so we are not concerned with the performance of 
the true posterior $\nu$ in terms of its 
concentration around the true value of the  unknown 
parameter $u$.

We make the following four contributions towards the above goal
in the setting where $\U \subseteq \R^d$ for $d \ge 1$:

\begin{enumerate}[label=(\roman*)]
    \item Writing $\W_p$ for $p=1,2$ to denote the 
    Wasserstein-$p$ distance (see \Cref{subsec:w1-pert}
    for definition) we prove a generic bound of the form 
    \begin{equation*}
        \W_1(\nu, \widehat{\nu} ) \le \Cstab \W_2(\mu, \widehat{\mu}),
    \end{equation*}
    where $\Cstab >0$ is a constant that depends on 
    regularity of the likelihood $\Phi$ and some moments 
    of $\mu$ and $\widehat{\mu}$. The precise result is given 
    in \Cref{thm:main-1-wasserstein-stability} building on  \cite{sprungk2020local,garbuno2023bayesian}.

    \item Consider any  map 
        $\widehat{T} \in \argmin_{T \in \widehat{\T}} 
        \W_2( T \# \eta, \mu^N),$
    where $\widehat{\T}$ denotes an approximation 
    class such as neural nets or polynomials, and 
    $\mu^N$ denotes the empirical measure 
    associated with $N$ i.i.d. samples from $\mu$. 
    We then prove, under appropriate assumptions, 
    that for any $\varepsilon > 0$ and with probability 
    $1 - C N^{-1/d}/\varepsilon$ (with some constant $C>0$) 
    it holds that
    \begin{equation*}
        \W_2(\widehat{\mu}, \mu) 
        \lesssim \inf_{T \in \widehat{\T}} 
        \| T - T^\dagger \|_{L^2_\eta} + \varepsilon,
    \end{equation*}
    where $\lesssim$ contains independent constants and 
    $T^\dagger$ is any map that satisfies $T^\dagger \# \eta = \mu$.
    The precise statement is given in \Cref{sec:prior-bounds}.

    \item Under further assumptions we extend the 
    above bound to the posteriors, showing that in the same 
    high-probability event, we have 
    \begin{equation*}
        \W_1(\widehat{\nu}, \nu) \lesssim 
        \inf_{T \in \widehat{\T}} 
        \| T - T^\dagger \|_{L^2_\eta} + \varepsilon + \delta
    \end{equation*}
    where $\delta >0$ is an additional bias term that encodes 
    the tail properties of the prior $\mu$. The precise statement 
    can be found in \Cref{sec:posterior-bounds}.

    \item Finally, we present a number of numerical experiments aimed 
    at verifying our theoretical bounds, in particular
    the result of item (i). We also present a nonlinear PDE inverse 
    problem with a generative prior that elucidates the 
    benefits of using generative priors to enhance the 
    performance of simple Markov chain Monte Carlo (MCMC) 
    methods for solving high-dimensional BIPs with 
    highly non-Gaussian priors.
    
\end{enumerate}

\subsection{Literature Review}\label{sec:literature-review}
We focus our review of the relevant literature to recent works 
on data-driven techniques for inverse problems. For a general 
overview of the Bayesian approach to inverse problems 
we refer the reader to \cite{stuart2010inverse,kaipio2005statistical}. 


\paragraph{Data-driven approaches to inverse problems}
Data-driven approaches to inverse problems became popular in the last 
decade alongside the wide adoption of ML techniques for imaging \cite{arridge2019solving}. Many data-driven approaches were widely 
used in inverse problems in the past, most notably, reduced order models 
\cite{cui2015data} and emulators \cite{kennedy2002bayesian} for 
likelihoods and forward maps. However, the more modern data-driven 
techniques are focused on the automated modeling of 
prior information. As opposed to classic inverse problems, where one 
often employs a smoothness prior, e.g. Tikhonov regularization 
\cite{vogel2002computational}, here side information or data is assumed 
to be available that can be used to enhance the prior. We already 
mentioned the three classes of data-driven priors: (1) 
learned priors \cite{bora2017compressed,lunz2018adversarial,mukherjee2024data, bohra2023bayesian,patel2022solution,ongie2020deep, dimakis2022deep, duff2024regularising, habring2022generative}; 
plug-and-play priors \cite{venkatakrishnan2013plug,zhang2017learning, zhang2022plug,laumont2022bayesian}; 
and simulation-based inference methods \cite{lueckmann2021benchmarking,song2021solving, baptista2024conditional,zammit2025neural}. 

Our work is focused on the analysis of algorithms for class (1) where we assume there exists, a priori, a data set of "typical" possible choices 
of the unknown parameter, e.g., a collection of MRI or CT images. 
Then the goal of data-driven inference is to build a bespoke 
prior for this data set, often using a generative model. 
While algorithmic and methodological developments in this space 
have attracted a lot of attention, to our knowledge the theoretical 
analysis of inverse problems with generative or adaptive priors 
remains limited. Previous works \cite{bora2017compressed, shah2018solving, jalal2021robust, adcock2025many} 
considered the "compressed sensing" setting where 
the forward map is linear (typically a random matrix) and 
the prior is generative, and give theoretical results and guarantees 
on the recovery of the latent representation of the parameters 
in the domain of the generative prior. While \cite{jalal2021robust, adcock2025many} have 
the closest analysis to our work, namely considering Wasserstein perturbations of the prior, their theory remains limited to linear forward maps and gives recovery guarantees for the parameter provided that the 
prior is "good enough" in a Wasserstein sense.
In comparison, our theory is simple and covers both linear and nonlinear problems as it directly quantifies the error in approximating the 
posterior measure due to learning the prior. We do not give recovery 
guarantees, however, our convergence results for the priors 
can be combined with the results in \cite{jalal2021robust,adcock2025many}
to further understand when their recovery guarantees hold.

Our analysis also has some bearing on class (3) methodologies, that is,
amortized or simulation-based inference \cite{baptista2024conditional, hosseini2025conditional, al2025error, alfonso2023generative, orozco2025aspire}. Broadly 
speaking, these models utilize existing or simulated data sets 
of pairs of parameters $u$ and data $y$ to directly 
learn a conditional generative model for the posterior $\nu$. 
While the prior $\mu$ is never explicitly identified or 
learned from data, it is implicitly learned through the data set 
often using similar generative models that are used to learn the 
priors in other types of data-driven inverse problems. Hence, we 
conjecture that our type of analysis can be extended to such 
conditional generative models, as is done in \cite{al2025error} for 
the case of triangular OT maps.

\paragraph{Error analysis of generative models}
The error analysis of generative models has attracted a lot of attention in recent years, most notably the works \cite{liang2021well, tang2023minimax, tahmasebi2024sample, huang2022error, singh2018nonparametric}, where the estimation 
error of GANs were studied, as well as the more recent 
body of work on the error analysis of flow 
and diffusion models \cite{albergo2025stochastic, benton2024error, huang2025convergence, marzouk2024distribution, pandey2025diffeomorphic}. The error analysis of transport models, such as those used in our work, has been done in various contexts: an abstract framework for the error analysis of generic transport models was developed in \cite{baptista2025approximation} while 
 \cite{irons2022triangular, wang2022minimax, zech2022sparse-I, zech2022sparse-II} studied triangular maps, also known as Knothe-Rosenblatt maps, and the works \cite{hutter2021minimax,divol2025optimal, chewi2024statistical, al2025error} developed the statistical analysis of optimal transport maps with 
 quadratic costs. Despite these works the broader approximation theory of transport-based generative modeling
 is far from complete and the combination of these results 
 with data-driven inverse problems, such as our work in this paper, remains largely open. 

\paragraph{Perturbation theory of BIPs}
The perturbation theory of BIPs is a cornerstone of our 
theoretical analysis of data-driven BIPs since such an 
analysis is inevitable when learned priors or likelihoods 
are involved. Historically, this perturbation analysis 
appeared in the context of the well-posedness of 
BIPs with early works focusing on Gaussian priors, see \cite{stuart2010inverse, dashti2015bayesian} and references within, followed by extensions to non-Gaussian priors 
\cite{hosseini2017well-I, hosseini2017well-II, sullivan2017well, latz2020well}. While these works 
primarily considered the perturbation of BIPs with respect to likelihood functions 
the paper \cite{sprungk2020local} extended that analysis to 
prior perturbations with respect to various metrics including the Wasserstein distances which 
are crucial to our analysis here. However, that analysis 
required a global Lipschitz assumption on the likelihood 
which does not hold for many practical problems including least squares problems. The manuscript \cite{garbuno2023bayesian} extended 
that analysis to bespoke optimal transport metrics that could handle only locally Lipschitz 
likelihoods and included a quantitative analysis of BIPs with generative priors that forms 
the foundation of our analysis in this paper.

\subsection{Outline}\label{sec:outline}
The rest of the article is organized as follows: 
\Cref{sec:wasserstein-perturbation} outlines a
perturbation analysis for Bayesian posterior measures with respect 
to the Wasserstein-1 distance, giving technical results that we use later in the paper; 
\Cref{sec:Theory} presents our main theoretical results 
and in particular the quantitative error bounds for 
learned priors and subsequent posteriors; 
\Cref{sec:Numerical Results} presents our numerical experiments that investigate and extend our 
theory; 
and \Cref{sec:Conclusion} outlines our concluding remarks 
and discussions.

\section{Posterior perturbations with 
respect to the Wasserstein-1 distance}\label{sec:wasserstein-perturbation}

Below we outline general perturbations results for BIPs that 
serve as the foundation for our error analysis later on.
We begin with preliminary results from \cite{garbuno2023bayesian, sprungk2020local}  in \Cref{subsec:prelims} which we then tailor 
to the case of the Wasserstein-1 distance in \Cref{subsec:w1-pert}.
Throughout this section we mostly consider the data 
$y$ to be fixed, but we comment on the dependence of 
our bounds on the data and how it can be 
controlled for additive noise models in \Cref{rem:data-bound}

\subsection{Perturbation bounds with respect to integral probability 
metrics}\label{subsec:prelims}

Consider separable Banach spaces $\U, \Y$ along with 
the spaces of Borel probability measures on them, denoted by 
$\PP(\U)$ and respectively $\PP(\Y)$.
For a continuous and positive cost $c: \U \times \U \to \R_{\ge 0}$ 
define the space
\begin{equation*}
    \Lip_c(\U) := \left\{ \psi : \U \to \R \; \big| \;  |\psi(u) - \psi(u')| \le 
    c(u,u'), \quad \forall u,u' \in \U \right\},
\end{equation*}
and further define the statistical divergence (also 
called an integral probability metric \cite{muller1997integral})
\begin{equation*}
    \D_c: \PP(\U) \times \PP(\U) \to \R_{\ge 0} \cup \{ + \infty\}, \qquad 
    \D_c( \gamma, \gamma') 
    := \sup_{\psi \in \Lip_c(\U)} \gamma(\psi) - \gamma'(\psi).
\end{equation*}
Note that $\D_c$ is not a true metric since it may not be definite, i.e., 
$\D_c(\gamma, \gamma') = 0$ does not imply that $\gamma = \gamma'$,
though it always satisfies the triangle inequality. 

For a measure $\mu \in \PP(\U)$, continuous functions 
$f: \U \to \R_{\ge 0}$, $h: \U \times \Y \to \R_{>0}$, 
and $\ell : \U \times \U \times \Y \to \R_{\ge 0}$, and 
 a locally bounded function $g: \Y \to \R_{>0}$, we define the 
class 
\begin{equation*}
\begin{aligned}
        \F(f, g, h, \ell,\mu) := 
    \Big\{ 
     \Phi: \U \times \Y \to \R \; & \Big| \; 
    \Phi(\cdot ; y) \in L^1(\mu), \text{ for all } y \in \Y, \\ 
    & \text{ and }- \log f(u)  - \log g (y) \le \Phi(u; y) \le  - \log h(u, y), \\ 
    & \text{ and } | \Phi(u; y) - \Phi(u';y) | \le \ell(u,v, y) \| u - u' \|_\U
    \Big\}.
\end{aligned}
\end{equation*}

With this setup we now recall the following result
from \cite{garbuno2023bayesian}
concerning the perturbation of posterior measures with respect to the 
priors: 

\begin{proposition}\protect\label{prop:post-prior-pert}
Consider priors $\mu, \widehat{\mu} \in \PP(\U)$, 
a likelihood $\Phi \in \F(f, g, h, \ell, \mu)$, 
and a continuous and positive cost $c: \U \times \U \to \R_{\ge 0}$. 
Further suppose that 
$f, f c(\cdot, 0) \in L^1(\mu)$ and $h(\cdot, y) \in L^1(\mu) \cap L^1(\widehat{\mu})$ for all $y \in \Y$. Then it holds that 
\begin{equation*}
    \D_c( \nu, \widehat{\nu} ) 
    \le g^2(y) \left[ 
    \frac{ \| f \|_{L^1(\mu)} + \|f c(\cdot, y) \|_{L^1(\mu)}}
    {\| h(\cdot, y) \|_{L^1(\mu)} \| h(\cdot, y) \|_{L^1(\widehat{\mu})}} \right] \D_{c_y}(\mu, \widehat{\mu}),
\end{equation*}
where $\nu, \widehat{\nu}$ are given by \eqref{bayes-rule} and 
\eqref{eq:bayes-rule-approx} respectively and 
$c_y: \U \times \U \to \R_{\ge 0}$ is a new cost given by the 
expression 
\begin{equation*}
    c_y( u, u') := \left[ 1 \vee c(u, 0) \vee c(u', 0) \right]
    \cdot \left[ f(u) \vee f(u') \right] 
    \cdot \left[ 1 \vee \ell(u, u', y) \right] 
    \cdot c(u, u'), \qquad \forall  u,u' \in \U.
\end{equation*}
\end{proposition}

The above proposition is central to our exposition in the rest of 
the paper as it allows us to control the error between 
two posterior measures in terms of the underlying priors although 
we note that the divergences on both sides of the inequality are 
different and the cost $c_y$ often leads to a stronger metric 
than the original one chosen to measure the posterior perturbations. 
We also emphasize that the cost $c_y$ depends on the observed data 
$y$ through the function $\ell$  which encodes the local Lipschitz constant of 
the likelihood $\Phi$.

\subsection{Perturbation bounds with respect to Wasserstein distances}\label{subsec:w1-pert}

We now tailor the result of \eqref{prop:post-prior-pert} to the 
setting of Wasserstein distances. 
For $p \in [1, +\infty)$
let us write $\PP^p(\U) := \{ \gamma \in \PP(\U) \; \mid \; 
\gamma( \| \cdot \|^p) < + \infty \}$, i.e., 
probability measures with bounded $p$ moments. Then 
we Recall the definition of 
the Wasserstein-$p$ distances from \cite{villani2009optimal}:
\begin{equation}\label{W_p definition}
\W_p(\gamma, \gamma') := \left(  \inf_{\pi \in \Pi( \gamma, \gamma')}
 \int_{\U \times \U} \| u - u' \|_\U^p \; \pi( \dd u, \dd u') \right)^{1/p}, \qquad \forall \gamma, \gamma' \in \PP^p(\U),
\end{equation}
where $\Pi(\gamma, \gamma') \subset \PP(\U \times \U)$
denotes the space of couplings of $\gamma$ and $\gamma'$, i.e., 
measures defined on the product space $\U \times \U$ 
whose first marginal coincides with $\gamma$ and their second marginals 
coincide with $\gamma'$: 
\begin{equation*}
    \Pi(\gamma, \gamma') := \left\{ 
        \pi \in \PP(\U \times \U) \; \Big| \; 
        \int_\U  \pi( \cdot , \dd u) = \gamma, \quad 
        \int_\U  \pi( \dd u, \cdot) = \gamma'
    \right\}.
\end{equation*}
By convention we set $\W_p(\gamma, \gamma') = +\infty$ if either 
measure does not belong to $\PP^p(\U)$. By the celebrated 
Kantorovich duality \cite[Thm.~5.10 and Case.~5.16]{villani2009optimal}
for the particular case $p=1$ we have the equivalent expression 
\begin{equation}\label{eq:w1-duality}
    \W_1(\gamma, \gamma') = 
    \sup_{\psi \in \Lip(\U)} \gamma(\psi) - \gamma'(\psi),
\end{equation}
where  we wrote $\Lip(\U) := \{ \psi: \U \to \R \mid |\psi(u) - \psi(u')| 
\le \| u - u'\|_\U \}$. In other words, $\W_1 = \D_c$ for the cost 
$c(u, u') = \| u - u' \|_\U$. Indeed, a straightforward 
calculation shows that for the cost functions 
$c(u,u') = \| u - u' \|_\U^p$ we have the inequality
\begin{equation}\label{eq:Dc-Wp-bound}
    \D_c(\gamma, \gamma') \le \W_p^p(\gamma, \gamma'),
\end{equation}
showing that the $\D_c$ divergences are weaker than the Wasserstein 
distances. More broadly, for any cost $c$ we also have the inequality 
\begin{equation}\label{eq:Dc-coupling-bound}
\D_c(\gamma, \gamma') \le \inf_{\pi \in \Pi(\gamma, \gamma')} \int_{\U \times \U} c(u, u') \pi( \dd u, \dd u') 
=: \W_c(\gamma, \gamma').
\end{equation}
We are now ready to present our first main theoretical results 
which allows us to control the $\W_1$ distance between 
posteriors in terms of the $\W_2$ distance between the priors. 
Our result extends those given in \cite[Sec.~5]{sprungk2020local} and 
 \cite[Sec.~4.3]{garbuno2023bayesian} to the case of likelihoods 
 that are not globally Lipschitz.

\begin{theorem}\label{thm:main-1-wasserstein-stability}
    Consider priors $\mu, \widehat{\mu} \in \PP^2( \U)$ 
    and a likelihood $\Phi \in \F(f, g, h, \ell, \mu)$. 
    Further suppose $(1 + \| \cdot \|_\U) f  \in L^1(\mu)$ 
    and $h(\cdot, y) \in L^1(\mu) \cap L^1(\widehat{\mu})$ for all 
    $y \in \Y$. Then it holds that 
    \begin{equation}\label{eq:wasserstein-posterior-prior-bound}
        \W_1(\nu, \widehat{\nu}) \le \Cstab(y) \W_2( \mu, \widehat{\mu}),
    \end{equation}
    where the constant $\Cstab(y)$ has the explicit form
    \begin{equation}\label{eq:Cstab-full-def}
    \begin{aligned}
                \Cstab(y) =  g^2(y) & \left[ 
    \frac{  \| (1 + \| \cdot \|_\U) f \|_{L^1(\mu)} }
    {\| h(\cdot, y) \|_{L^1(\mu)} \| h(\cdot, y) \|_{L^1(\widehat{\mu})}} \right] \\
    &\times 
    \left(\int_{\U \times \U} [1\vee \ell(u,v,y)]^2 \cdot [f(u) \vee f(v)]^2\cdot [1 \vee \|u\|_\U \vee \|v \|_\U]^2\pi^\star(\dd u,\dd v) \right)^{\frac{1}{2}}.
    \end{aligned}
    \end{equation}
    Here $\pi^\star$ denotes the Wasserstein-2 optimal coupling between $\mu$
    and $\widehat{\mu}$.
\end{theorem}




\begin{proof}
Since $\Phi\in \mathcal{F}(f,g,h,l, \mu)$, and $f, f \cdot \| \cdot \| \in L^1(\mu)$, by \Cref{prop:post-prior-pert} along with \eqref{eq:w1-duality}, we obtain the bound
\begin{equation}\label{prop:post-prior-W1}
    \W_1(\nu,\widehat{\nu}) \leq \frac{g^2(y)[\|f\|_{L^1(\mu)}+\|fc(\cdot,0)\|_{L^1(\mu)}}{\|h(\cdot ,y)\|_{L^1(\mu)}\|h(\cdot ,y)\|_{L^1(\widehat{\mu})}} \mathcal{D}_{c_y}(\mu,\widehat{\mu}),
\end{equation}
where $c_y$ is defined as:
\begin{align*}
c_y(u,v) = \left[ 1 \vee \| u \|_\U \vee \| u'\|_\U \right]
    \cdot \left[ f(u) \vee f(u') \right] 
    \cdot \left[ 1 \vee \ell(u, u', y) \right] 
    \cdot \|u-v \|_\U.
\end{align*}
We can simplify \eqref{prop:post-prior-W1} by observing 
$\| f \|_{L^1(\mu)} + \| f  \| \cdot \|_\U \|_{L^1(\mu)} = \| (1 + \| \cdot \|_\U)  f \|_{L^1(\mu)}$ since $f$ is positive. 
Moreover, an application of \eqref{eq:Dc-coupling-bound} yields 
$\D_{c_y}(\mu, \widehat{\mu}) \le \W_{c_y} (\mu, \widehat{\mu}).$
We now bound the right hand side of this expression by the Wasserstein-2 distance 
between $\mu$ and $\widehat{\mu}$.
By the Cauchy-Schwarz inequality, 
for any $\pi \in \prod(\mu,\widehat{\mu})$ we have the bound
\begin{align*}
& \int_{\U \times \U}c_y(u,v)\pi(\dd u,\dd v) \\
& \qquad \leq  \left(\int_{\U \times \U} [1\vee \ell(u,v,y)]^2 \cdot [f(u) \vee f(v)]^2\cdot [1 \vee \|u\|_\U \vee \|v \|_\U]^2\pi(\dd u,\dd v) \right)^{\frac{1}{2}} \\
& \qquad \qquad \times \left(\int_{\U \times \U}|u-v|^{2}\pi(\dd u,\dd v)\right)^{\frac{1}{2}}.
\end{align*}
Let $\pi^\star(u,v) = \argmin_{\pi \in \prod(\mu,\widehat{\mu})}\int_{\U \times \U}|u-v|^{2}\pi(\dd u,\dd v)$ be the optimal Wasserstein-2 coupling and apply the above bound 
to get 
\begin{align*}
    \W_{c_y}(\mu, \widehat{\mu}) & \le  \int_{\U \times \U} c_y(u, v) \pi^\star(\dd u,\dd v) \\
    &\leq
         \left(\int_{\U \times \U} [1\vee \ell(u,v,y)]^2 \cdot [f(u) \vee f(v)]^2\cdot [1 \vee \|u\|_\U \vee \|v \|_\U]^2\pi^\star(\dd u,\dd v) \right)^{\frac{1}{2}} \W_2(\mu, \widehat{\mu}). 
\end{align*}
Combining this bound with  \eqref{prop:post-prior-W1} completes the proof.
{}
\end{proof}
Let us now outline a corollary to the above result for the case of 
additive Gaussian noise models that are commonly used in practice.
\begin{corollary}\label{coro:likelihood in F-class}
Consider priors $\mu, \widehat{\mu}$ and the quadratic likelihood
potential
\begin{equation}\label{eq:additive-likelihood}
    \Phi(u; y) = \frac{1}{2\sigma^2} \| F(u) - y \|_\Y^2,
\end{equation}
arising from the model $y = F(u) + \xi$, $\xi \sim N(0, \sigma^2 I)$,
where $F: \U \to \Y$ is locally Lipschitz such that
\begin{equation*}
    \| F(u) - F(u') \|_\Y \le \ell_F(u, u') \| u - u' \|_\U, \quad \forall u,u' \in \U,
\end{equation*}
for some function $\ell_F: \U \times \U \to \R_{\ge 0}$. Further suppose 
$\exp( - \frac{1}{\sigma^2} \| F( \cdot )\|_\Y^2 ) \in L^1(\mu) \cap L^1(\widehat{\mu})$. 
Then \eqref{eq:wasserstein-posterior-prior-bound} holds with 
\begin{equation}\label{eq:cstab-additive-noise}
\begin{aligned}
    &\Cstab(y) = \left[ 
    \frac{\exp\left( \frac{2}{\sigma^2}\| y\|_\Y^2 \right) \big(1 + \mu(\|\cdot\|_\U)\big) }
    {\mu \left( \exp(-\frac{1}{\sigma^2}\|F(\cdot)\|^2_\Y) \right)  \widehat{\mu} \left(\exp(-\frac{1}{\sigma^2}\|F(\cdot)\|_\Y^2 ) \right) } \right] \\
    & \qquad \qquad \times 
    \Bigg(\int_{\U \times \U} 
    \left[1\vee \left(\frac{1}{2\sigma^2} \ell_F(u, v)(\|F(u)\|_\Y+\|F(v)\|_\Y+2\|y\|_\Y )\right)^2
    \right] \\
    &\qquad \qquad \qquad \cdot \Big[1 \vee \|u\|_\U \vee \|v \|_\U \Big]^2\pi^\star(\dd u,\dd v) \Bigg)^{\frac{1}{2}}.
\end{aligned}
\end{equation}
\end{corollary}
\begin{proof}
Since $\Phi$ is positive we can take $f = g =1$. Moreover, by 
the triangle inequality 
    \begin{align*}
        \frac{1}{2\sigma^2} \| F(u) - y \|_\Y^2 \leq \frac{1}{\sigma^2} (\|F(u)\|_\Y^2 + \|y\|_\Y^2) = -\log h(u,y).
    \end{align*}
    To show $\Phi$ is locally Lipschitz, consider
    \begin{align*}
       2 \sigma^2  \|\Phi(u;y) - \Phi(u', y)\| &=  \|F(u)-y\|_\Y^2 - \|F(u')-y\|_\Y^2  \\
        &= (\|F(u)-y\|_\Y - \|F(u')-y\|_\Y)(\|F(u)-y\|_\Y+\|F(u')-y\|_\Y) \\
        &\leq  \|F(u) - F(u')\|_\Y(\|F(u)-y\|_\Y+\|F(u')-y\|_\Y) \\
        &\text{(By reverse triangle inequality)}\\
        &\leq  \ell_F(u, u') \| u - u' \|_\U (\|F(u)\|_\Y+\|F(u')\|_\Y+2\|y\|_\Y)\\
        &\text{(By triangle inequality)}
    \end{align*}
    Since $F$ is locally Lipschitz 
    then $F$ is also locally bounded which implies that 
    $\Phi$ is locally Lipschitz with 
        $\ell(u, u', y) = \frac{1}{2\sigma^2} \ell_F(u, u') (\|F(u)\|_\Y+\|F(u')\|_\Y+2\|y\|_\Y).$
    An application of \Cref{thm:main-1-wasserstein-stability} and 
    simplifying the expression for the constants yields the desired result.
\end{proof}

\begin{remark}\label{rem:data-bound}
We note that the constant $\Cstab$ in \eqref{eq:cstab-additive-noise} is dependent on the data $y$. It is helpful to try to simplify this expression 
to further highlight the dependence of the stability constant of the 
data as this is often crucial in the study of the consistency and 
sensitivity of BIPs to noise in the data \cite{owhadi2015brittleness}. To this end, assuming suitable integrability conditions we
can write 
\begin{align*}
    &\int_{\U \times \U} \left[1\vee \left(\frac{1}{2\sigma^2} \ell_F(u, v) (\|F(u)\|_\Y+\|F(v)\|_\Y+2\|y\|_\Y)\right)^2 \right] \cdot \left[1 \vee \|u\|_\U \vee \|v \|_\U \right]^2\pi^\star(\dd u,\dd v) \\
    &\lesssim \int_{\U \times \U} \left[1 + \frac{1}{\sigma^4} \ell_F^2(u, v) (\|F(u)\|_\Y+\|F(v)\|_\Y+2\|y\|_\Y)^2 \right] \cdot 
    \left[1 \vee \|u\|_\U^2 \vee \|v \|_\U^2 \right]\pi^\star(\dd u,\dd v) \\
    &\lesssim \int_{\U \times \U} \left[1 + \frac{1}{\sigma^4} 
    \ell_F^2(u, v) (\|F(u)\|_\Y+\|F(v)\|_\Y)^2+\frac{1}{\sigma^4} \ell^2_F(u, v)^2\|y\|_\Y^2 \right] \\ 
    & \qquad \qquad \cdot \left[1 + \|u\|_\U^2 + \|v \|_\U^2 \right]\pi^\star(\dd u,\dd v)\\
    &\lesssim \int_{\U \times \U} \left[1 + \frac{1}{\sigma^4} 
    \ell_F^2(u, v) (\|F(u)\|_\Y^2+\|F(v)\|_\Y^2) \right] \cdot \left[1 + \|u\|_\U^2 + \|v \|_\U^2 \right]\pi^\star(\dd u,\dd v)\\
    & \qquad \qquad +\left[\frac{1}{\sigma^4} \|y\|_\Y^2 \right]\int_{\U \times \U} \ell_F(u, v)^2\cdot [1 + \|u\|_\U^2 + \|v \|_\U^2]\pi^\star(\dd u,\dd v) \\
    & =: \left(1 + \frac{1}{\sigma^4}\right) 
    \left(\alpha + \alpha' \| y\|_\Y^2\right),
\end{align*}
where we repeatedly used the identity $(a+ b)^2 \le 2( a^2 + b^2)$ and 
$\lesssim$ contains a universal constant that is independent of $\sigma$.
We note that $\alpha, \beta >0$ are constants that depend only on the 
growth rate of $F$ and $\ell_F$ and effectively translate to moment 
conditions on $\pi^\star$. To this end, we can explicitly 
characterize the dependence of $\Cstab$ in terms of $y$ as
\begin{equation*}
    \Cstab(y) \lesssim  \alpha'' \left(1 + \frac{1}{\sigma^2} \right) \exp \left( \frac{2}{\sigma^2} \| y \|_\Y^2 \right) ( 1+ \| y\|_\Y ), 
\end{equation*}
where once again $\alpha''>0$ depends only on the moments of $\pi^\star$. Let us now write $\upsilon \in \PP(\Y)$ for the marginal distribution of $y$ under the prior, i.e., $\upsilon = F\# \mu \ast N(0, \sigma^2 I)$. 
Then we have the bound 
\begin{equation*}
\EE_{y \sim \upsilon} \W_1(\nu, \widehat{\nu})
\lesssim  \alpha'' \left(1 + \frac{1}{\sigma^2} \right) 
\upsilon\left( ( 1+ \| y\|_\Y ) \exp \left( \frac{2}{\sigma^2} \| y \|_\Y^2 \right)  \right) \W_2(\mu, \widehat{\mu}).
\end{equation*}
Thus the average $\W_1$ distance between the posteriors is controlled 
by the prior $\W_2$ distance provided that the marginal $\upsilon$ 
has sufficiently light tails and the prior has sufficient 
bounded moments so that $\alpha'' < +\infty$.
\end{remark}


\section{Error analysis of BIPs with generative priors}\label{sec:Theory}
In this section we build on the perturbation bounds in \Cref{sec:wasserstein-perturbation} to give quantitative bounds for posteriors that 
arise from generative priors that can be modeled as  transport maps. 
We begin by controlling the errors of the generative priors in \Cref{sec:prior-bounds} followed by the extension to error bounds for the posteriors in \Cref{sec:posterior-bounds} that constitute our main results. We note that throughout this section we assume that the data 
$y$ is fixed and only consider the errors due to the 
prior generative models. At various points we provide further remarks 
on how one could extend our results to contain perturbations 
of the data $y$.

\subsection{Prior bounds}\label{sec:prior-bounds}

We begin by deriving error bounds for the prior measures, i.e., 
a bound on $\W_2(\mu, \widehat{\mu})$ where take $\widehat{\mu}$
to be a generative prior and $\mu$ to be the ground truth prior. 
To this end, we consider the model 
\begin{equation}\label{eq:prior-generative-model-M-N}
    \widehat{\mu}^{N,M} := \widehat{T}^{N,M} \# \eta, \qquad 
    \widehat{T}^{N,M} 
    := \argmin_{T \in \widehat{\T}} \W_2( T \# \eta^M, \mu^N),
\end{equation}
where $\eta^M, \mu^N$ are empirical approximations to 
$\eta$ and $\mu$ with $M$ and $N$ samples respectively 
and $\widehat{\T}$ is an approximation class of maps $T: \U \to \U$
that models our generative model class, such as neural nets of a 
certain size, polynomials, or  reproducing kernel Hilbert spaces 
(RKHSs). Our goal in this subsection is to bound $\W_2(\mu, \widehat{\mu}^{M,N})$.

 To simplify our analysis, we first consider the 
slightly simpler problem 
\begin{equation}\label{eq:prior-generative-model-N}
    \widehat{\mu}^N := \widehat{T}^N \# \eta, \qquad 
    \widehat{T}^N := \argmin_{T \in \widehat{\T}} D(T \# \eta, \mu^N),
\end{equation}
where we only use the empirical prior $\mu^N$ to train 
the generative model and use the exact reference $\eta$\footnote{Indeed practical generative models such 
as normalizing flows \cite{kobyzev2021normalizing} are trained 
this way.}. Throughout this section we assume $\widehat{T} \subset L^2(\eta)$ and further assume there exists a ground truth map 
$T^\dagger \in L^2(\eta)$ such that $\mu = T^\dagger \#  \eta$. 
Finally, we take $d \ge 4$ to simplify our statements later on.

Our first result is a technical lemma that allows us 
to decompose the error of $\widehat{\mu}^N$ into an 
approximation bias, due to the choice of $\widehat{\T}$
and a stochastic error due to finite samples. 

\begin{lemma}\label{lem:W2-oracle-inequality}
Consider \eqref{eq:prior-generative-model-N} and define 
\begin{equation}\label{eq:hatT-dagger}
    \widehat{T}^\dagger := \argmin_{T \in \widehat{\T}} 
    \| T - T^\dagger \|_{L^2(\eta)}. 
\end{equation}
Then it holds that
$\W_2(\widehat{\mu}^N, \mu)  \le \| \widehat{T}^\dagger - T^\dagger \|_{L^2(\eta)} + 2 \W_2(\mu, \mu^N).$
\end{lemma}
\begin{proof}
    We can then write the sequence of inequalities
    \begin{equation}\label{eq:oracle-bounding-chain}        
    \begin{aligned}
        \W_2(\widehat{T}^N \# \eta, \mu) 
        &\leq \W_2(\widehat{T}^N \# \eta, \mu^N)+\W_2(\mu, \mu^N) \quad \text{(By the Triangle Inequality)}\\
        &\leq \W_2(\widehat{T}^\dagger \# \eta, \mu^N)+\W_2(\mu, \mu^N) \quad \text{(By optimality of $\widehat{T}^N$)} \\
        &\leq \W_2(\widehat{T}^\dagger\# \eta, \mu)+ 2\W_2(\mu, \mu^N)\quad \text{(By the Triangle Inequality)}\\
        &=\W_2(\widehat{T}^\dagger \# \eta,T^\dagger \# \eta)+ 2\W_2(\mu, \mu^N).
    \end{aligned}
    \end{equation}
    By \cite[Thm.~3.1]{baptista2025approximation} we have, for any 
    pair of maps $T, T' \in L^2(\eta)$,
    \begin{equation}\label{eq:W2-stability}
        \W_2( T \# \eta, T' \# \eta) \le \| T - T'\|_{L^2(\eta)},
    \end{equation}
    and so the result follows by applying \eqref{eq:W2-stability}
    to $\W_2(\widehat{T}^\dagger \# \eta, T^\dagger \# \eta)$.
\end{proof}

\begin{remark}
    We note that \Cref{lem:W2-oracle-inequality} is analogous 
    to the oracle inequalities derived in \cite{liang2021well}. 
    While the proof is simple, it has wide applications 
    as it allows us to deal with the stochastic part of the error 
    with the term $\W_2(\mu, \mu^N)$ which is a well-understood object
    in the empirical analysis of optimal transport. 
    Indeed,  such decompositions can be obtained for other metrics and divergences that satisfy a stability property, i.e., 
    $\D(T \# \eta, T' \# \eta) \lesssim \| T - T' \|_{L^2(\eta)}.$ 
    However, the key to \Cref{lem:W2-oracle-inequality} is that 
    $\widehat{T}^N$ is obtained by minimizing the same divergence $\D$
    that is also used for the error analysis. 
\end{remark}

It remains for us to bound $\W_2(\mu, \mu^N)$ to complete our 
error analysis. To this end, we recall the following 
well-known result on the empirical approximation of 
Wasserstein distances. 

\begin{lemma}[{\cite[Thm.~1]{fournier2015rate}}]
\label{lem:wasserstein-empirical-approximation-rate}
Suppose $\mu \in \PP(\R^d)$ and $d > 4$  and that 
$\chi_q(\mu):= \left( \int_{\R^d} |u |^q \mu( \dd u) \right)^{1/q} < +\infty$ for some 
$q > 2$. Then there exists a constant $C>0$, depending only on $d$, 
such that for all $N \ge 1$, 
\begin{equation*}
    \EE \; \W^2_2(\mu, \mu^N) \le C \chi_q^2(\mu) \left( N^{-2/d} - N^{-(q - 2)/q} \right),
\end{equation*}
where the expectation is with respect to the empirical samples from 
$\mu$.
\end{lemma}

\begin{remark}
Note that 
\cite[Thm.~1]{fournier2015rate} covers the case 
 $d < 4$ but this mostly leads to unwieldy expressions for us that have little bearing on the main implications of our theory since $d >4$ is reasonable for most practical inverse problems. In fact, 
henceforth we assume that $q > 2d/(d-2)$ so that the term 
$N^{-(q - 2)/q}$ can be dropped and we have the simpler rate 
$N^{-2/d}$.
\end{remark}

 We are now ready to give our first error bound for 
generative priors.

\begin{proposition}\protect\label{prop:prior-error-bound-N}
Suppose $d> 4$, $\mu \in \PP(\R^d)$, and $\chi_q(\mu) < +\infty$ for some $q > 2d/(d-2)$. Then there exists a constant $C> 0$, independent of $N$,  such that for any $\epsilon >0$ it holds with 
probability $1 - C \chi_q(\mu) N^{-1/d}/\epsilon$ that
\begin{equation}\label{eq:prob-prior-bound}
    \W_2(\mu, \widehat{\mu}^N) \le \| \widehat{T}^\dagger - T^\dagger \|_{L^2(\mu)} + \epsilon.
\end{equation}
\end{proposition}

\begin{proof}
     \Cref{lem:W2-oracle-inequality} and \Cref{lem:wasserstein-empirical-approximation-rate} together give  
    \begin{equation}\label{eq:expectation-prior-bound}
        \EE \: \W_2(\mu, \widehat{\mu}^N) \le \| \widehat{T}^\dagger - T^\dagger\|_{L^2(\eta)} + C \chi_q(\mu) N^{-1/d}.
    \end{equation}
    Absorbing the independent constants into  $C$ and applying
     Markov's inequality yields the result. 
\end{proof}

\begin{remark}
    At this point one can directly work with \eqref{eq:expectation-prior-bound} as opposed to the probabilistic statement \eqref{eq:prob-prior-bound}. However, the latter becomes crucial when we analyze 
    the error of posterior measures due to technicalities that arise
    in controlling the perturbation of the $\Cstab$ constants that appeared in \Cref{sec:posterior-bounds}.
\end{remark}

Following a similar approach for the proof of \Cref{prop:prior-error-bound-N} we can obtain a similar result for the 
case where the reference $\eta$ is replaced with an empirical 
approximation $\eta^M$. 

\begin{lemma}\label{lem:Transport-map-Lipschitz}
Suppose  $T \in L^2(\eta)$ is globally Lipschitz with constant $\ell_{T}$, then,
for any pair of measures $\eta, \eta' \in \PP(\U)$ it holds that
$\W_2(T\#\eta, T\#\eta') \leq \ell_{T} \W_2(\eta, \eta')$.
\end{lemma}
\begin{proof}
Let $\pi^\star$ denote the optimal $\W_2$ coupling between $\eta$ and $\eta'$. Then using 
the Lipschitz assumption on $T$ we can write,
    \begin{align*}
        \W_2^2(T\#\eta, T\#\eta') &=   \inf_{\pi \in \Pi( \eta, \eta')}
    \int_{\U \times \U} \| T(x) - T(y) \|_\U^2 \; \pi( \dd x, \dd y) \\
     &\leq \int_{\U \times \U} \| T(x) - T(y) \|_\U^2 \; \pi^\star( \dd x, \dd y)  = \ell_T^2 \int_{\U \times \U} \| x - y \|_\U^2 \; \pi^\star( \dd x, \dd y)
    \end{align*}
    {}
\end{proof}

\begin{theorem}\protect\label{prop:prior-error-bound-M-N}
    Suppose $d > 4$, $\mu, \eta \in \PP(\R^d)$, and $\chi_q(\mu), \chi_q(\eta) < +\infty$ for some $q > 2d/(d-2)$. 
    Further suppose that the elements of $\widehat{\T}$ are uniformly globally 
    Lipschitz with constant $\ell_{\widehat{\T}} > 0$. 
    Then     there exists a constant $C> 0$, independent of $N$ and $M$, 
    such that for any $\epsilon >0$ it holds with probability 
    $(1 - \ell_{\widehat{\T}} C \chi_q(\eta) M^{-1/d}/\epsilon) (1 - C \chi_q(\mu) N^{-1/d}/\epsilon)$ that 
    \begin{equation*}
        \W_2(\mu, \widehat{\mu}^{M,N}) \le \| \widehat{T}^\dagger 
        - T^\dagger \|_{L^2(\eta)} + \epsilon.
    \end{equation*}
\end{theorem}

\begin{proof}
    Proceeding analogously to the proof of \Cref{lem:W2-oracle-inequality}
    we can write 
    \begin{align*}
    &\W_2(\widehat{\mu}^{N,M}, \mu)  = \W_2(\widehat{T}^{N,M}\#\eta, \mu)\\
    &\leq \W_2(\widehat{T}^{N,M}\#\eta, \widehat{T}^{N,M}\#\eta^M) + \W_2(\widehat{T}^{N,M}\#\eta^M, \mu)\quad\text{(By the Triangle Inequality)}\\
    &\leq \ell_{\widehat{\T}}\W_2(\eta, \eta^M) + \W_2(\widehat{T}^{N,M}\#\eta^M, \mu^N)+\W_2(\mu^N, \mu) \\
    &\text{(By the Triangle Inequality, \Cref{lem:Transport-map-Lipschitz})}\\
    &\leq \ell_{\widehat{\T}}\W_2(\eta, \eta^M)+ \W_2(\widehat{T}^{\dagger}\#\eta^M, \mu^N)+\W_2(\mu^N, \mu)\quad\text{(By optimality of $\widehat{T}^{N,M}$)}\\
    &\leq \ell_{\widehat{\T}}\W_2(\eta, \eta^M)+ \W_2(\widehat{T}^{\dagger}\#\eta^M, \widehat{T}^{\dagger}\#\eta)+\W_2(\widehat{T}^{\dagger}\#\eta,\mu^N)+\W_2(\mu^N, \mu) \\
    &\text{(By the Triangle Inequality)}\\
    &\leq 2\ell_{\widehat{\T}}\W_2(\eta, \eta^M)+ \W_2(\widehat{T}^{\dagger}\#\eta,\mu) +2\W_2(\mu^N, \mu)\\
    &\text{(By the Triangle Inequality and applying \Cref{lem:Transport-map-Lipschitz} again)}\\
\end{align*}
The second term can be controlled by \eqref{eq:W2-stability} while we can apply \Cref{lem:wasserstein-empirical-approximation-rate} to the first and third terms and absorbing the lipschitz constant to get, 
\begin{align*}
    \EE \: \W_2(\eta, \eta^M) &\leq C \chi_q(\eta)M^{-1/d}\qquad
    \EE \: \W_2(\mu^N, \mu) \le C \chi_q(\mu) N^{-1/d},
\end{align*}
Applying Markov's inequality to each term, we have that 
\begin{equation*}
    2\ell_{\widehat{\T}}\W_2(\eta, \eta^M) + 2\W_2(\mu^N, \mu) < \epsilon
\end{equation*}
with probability $(1 - 4 \ell_{\widehat{\T}} C \chi_q(\eta) M^{-1/d}/\epsilon) (1 - 2C \chi_q(\mu) N^{-1/d}/\epsilon)$. Absorbing the independent constants into $C$ completes the proof.
\end{proof}
\begin{remark}
    We note that the uniform Lipschitz assumption on the approximation class 
    $\widehat{\T}$ may appear restrictive but it is feasible in practical 
    applications by, for example, weight clipping strategies or regularization 
    for neural nets \cite{virmaux2018lipschitz, bungert2021clip}.
\end{remark}
\begin{remark}
    We note that the factor $\chi_q(\widehat{T}^\dagger \# \eta)$ 
    is innocuous and can be approximated by $\chi_q(\mu)$ provided 
    that $\widehat{T}^\dagger$ is close to $T^\dagger$ in 
    $L^q(\eta)$. 
    To be precise, we observe that 
    \begin{equation*}
    \begin{aligned}
    \chi_q^q(\widehat{T}^\dagger \# \eta) 
        = \int_{\R^d} | u|^q \widehat{T}^\dagger \# \eta(\dd u) 
        = \int_{\R^d} |\widehat{T}^\dagger (u) |^q \eta(\dd u) 
        = \| \widehat{T}^\dagger \|_{L^q(\eta)}^q.        
    \end{aligned}
    \end{equation*}
    Similarly $\chi_q(\mu) = \| T^\dagger \|_{L^q(\eta)}$.
    Thus, $\chi_q(\widehat{T}^\dagger \# \eta) \le 
    \| \widehat{T}^\dagger - T^\dagger \|_{L^q(\eta)} + \chi_q(\mu)$. 
\end{remark}

\subsection{Posterior bounds}\label{sec:posterior-bounds}
Combining the error bounds of \Cref{sec:prior-bounds} with 
the perturbation analysis of \Cref{subsec:w1-pert}
we can present high probability error bounds for the
posteriors of generative priors. We first treat 
priors with bounded support in \Cref{sec:posterior-bound-bounded-support}
followed by the unbounded support case in \Cref{sec:posterior-bounds-unbounded-support}.

Throughout this section we consider the posterior measure 
\begin{equation*}
    \frac{\dd \widehat{\nu}^{N,M}}{ \dd \widehat{\mu}^{N,M}}
    = \frac{1}{Z^{N,M}(y)} \exp( - \Phi(u; y) ), \qquad Z^{N,M}(y) 
    = \widehat{\mu}^{N,M}( \exp( - \Phi(u; y)),
\end{equation*}
as our approximation to the ground truth posterior $\nu$ from \eqref{bayes-rule} with the prior approximated by a generative model
with $N$ samples from $\mu$ and $M$ samples from $\eta$.

\subsubsection{The bounded support case}
\label{sec:posterior-bound-bounded-support}

In this section we assume that $\U \subset \R^d$ is a 
bounded set and the prior $\widehat{\mu}$ and its approximation 
$\widehat{\mu}^{M,N}$ are both supported on $\U$.
We begin with a lemma for controlling the stability constant 
$\Cstab$ from \Cref{thm:main-1-wasserstein-stability} and 
\Cref{coro:likelihood in F-class}. 

\begin{lemma}\label{lem:Cstab-control-bounded-case}
    Suppose $\U$ is bounded, $\mu, \widehat{\mu} \in \PP^2(\U)$ 
    and $\Phi \in \F(f, g, h, \ell, \mu)$ such that 
    $(1 + \| \cdot \|_\U) f \in L^1(\mu)$ and 
    $h(\cdot, y)$ is globally $\ell_h(y)$-Lipschitz, i.e., 
    \begin{equation*}
        |h(u,y) - h(u',y)| \le \ell_h(y) \| u - u'\|_\U.
    \end{equation*}
    Further suppose $\W_2(\mu, \widehat{\mu}) < \epsilon$ for 
    a small constant $\epsilon >0$. Then it holds that
    \begin{equation}\label{eq:cstab-lemma-display}
    \begin{aligned}
        \Cstab(y) \leq  & \left[ 
    \frac{  g^2(y) \| (1 + \| \cdot \|_\U ) f \|_{L^1(\mu)} }
    {\| h(\cdot, y) \|_{L^1(\mu)} (\| h(\cdot, y) \|_{L^1(\mu)}-\ell_h(y) \epsilon)} \right] \\
    &\times  [1 \vee  \Diam(\U)]
    \left(\int_{\U \times \U} [1\vee \ell(u,v,y)]^2 \cdot [f(u) \vee f(v)]^2 \; \pi^\star(\dd u,\dd v) \right)^{\frac{1}{2}},
    \end{aligned}
    \end{equation}
    where $\Diam(\U) = \sup_{u,u' \in \U} \| u - u' \|_\U$ is the 
    diameter of $\U$.
\end{lemma}

\begin{proof}
   Since $\U$ is bounded we can write 
   \begin{equation}\label{eq:Cstab-bounded-domain-case}
   \begin{aligned}
    &\left(\int_{\U \times \U} [1\vee \ell(u,v,y)]^2 \cdot [f(u) \vee f(v)]^2\cdot [1 \vee \|u\|_\U \vee \|v \|_\U]^2\pi^\star(\dd u,\dd v) \right)^{\frac{1}{2}}, \\
    &\leq \left([1 \vee \Diam(\U)^2] \int_{\U \times \U} [1\vee \ell(u,v,y)]^2 \cdot [f(u) \vee f(v)]^2\cdot  \pi^\star(\dd u,\dd v) \right)^{\frac{1}{2}}, \\
    & = [ 1 \vee \Diam(\U)] \left( \int_{\U \times \U} [1\vee \ell(u,v,y)]^2 \cdot [f(u) \vee f(v)]^2\cdot  \pi^\star(\dd u,\dd v) \right)^{\frac{1}{2}}.
    \end{aligned}
   \end{equation}
Now let $\pi^{\star \star}$ be the optimal $\W_1$ coupling between 
$\mu$ and $\widehat{\mu}$ and integrate both sides of the Lipschitz condition $h(u, y) - h(u', y) \le \ell_h(y) \| u -u'\|_\U$ with respect 
to this coupling to obtain 
\begin{equation*}
    \mu( h(\cdot, y) ) - \widehat{\mu}(h(\cdot, y)) 
    \le \ell_h(y) \W_1(\mu, \widehat{\mu}).
\end{equation*}
But $h$ is non-negative and so $\mu(h(\cdot, y)) = \| h(\cdot, y) \|_{L^1(\mu)}$ and similarly $\widehat{\mu}(h(\cdot, y)) = \| h(\cdot, y) \|_{L^1(\widehat{\mu})}$. 
Furthermore, $\W_1(\mu, \widehat{\mu}) \le \W_2(\mu, \widehat{\mu})$
(see for example \cite[Rem.~6.6]{villani2009optimal}, and so
 we obtain the lower bound $\| h(\cdot, y) \|_{L^1(\mu)} - \ell_h(y) \epsilon \le \| h(\cdot, y) \|_{L^1(\widehat{\mu})}$. 
Substituting this bound alongside \eqref{eq:Cstab-bounded-domain-case}
into \eqref{eq:Cstab-full-def} completes the proof. 
{}
\end{proof}
We need a second technical lemma that is useful in the proof of our main theorem below. 

\begin{lemma}\label{lem:wasserstein-expectation-bound}
Suppose $\varphi: \U \times \U \to \R$ is $\ell_\varphi$-Lipschitz with
respect to the metric $\varrho((u,v), (u',v')) := \| u -u' \|_\U + \| v - v' \|_\U$, i.e., 
\begin{equation*}
    | \varphi(u, v)  - \varphi(u', v') | \le \ell_\varphi \left( \| u - u'\|_{\U} 
    + \| v - v' \|_\U \right).
\end{equation*}
Suppose $\mu, \widehat{\mu} \in \PP(\U)$ such that $\W_2(\mu, \widehat{\mu}) < \epsilon$ and let $\pi^\star$ denote the corresponding 
optimal $\W_2$ coupling. Then it holds that 
\begin{equation*}
    \left| \int_{\U} \varphi(u,u) \mu(\dd u)  - \int_{\U \times \U} 
    \varphi(u,v) \pi^\star(\dd u, \dd v) \right| 
    \le \ell_\varphi \epsilon.
\end{equation*}
\end{lemma}
\begin{proof}
    By the Lipschitz assumption on $\varphi$ and Cauchy-Schwarz
    we have that 
    \begin{equation*}
       \begin{aligned}
           \left| \int_{\U} \varphi(u,u) \mu(\dd u)  - \int_{\U \times \U} 
    \varphi(u,v) \pi^\star(\dd u, \dd v) \right| 
    &\le \int_{\U \times \U} | \varphi(u, u) - \varphi(u, v) | \pi^\star ( \dd u, \dd v) \\ 
    & \le \ell_\varphi \int_{\U \times \U} \| u - v \|_\U \pi^\star(\dd u, \dd v)  \\
    & \le \ell_\varphi \left( \int_{\U \times \U} \| u - v \|^2_\U \pi^\star(\dd u, \dd v) \right)^{1/2} = \ell_\varphi \epsilon.
       \end{aligned} 
    \end{equation*}
    {}
\end{proof}

We are now ready to combine
 \Cref{lem:Cstab-control-bounded-case} with 
\Cref{thm:main-1-wasserstein-stability} and \Cref{prop:prior-error-bound-M-N} to obtain our first quantitative error bound for 
the posteriors. 

\begin{theorem}\label{thm:posterior-error-bound-bounded-case}
    Suppose \Cref{prop:prior-error-bound-M-N}, and \Cref{lem:Cstab-control-bounded-case} are satisfied and define the function 
    \begin{equation}\label{def:varphi}
        \varphi(u,v, y):= \big( 1 + \ell(u, v, y)^2 \big) ( f(u)^2 + f(v)^2 ). 
    \end{equation}
    Suppose $\varphi(\cdot, \cdot, y)$
    is $\ell_\varphi(y)$-Lipschitz with respect to the metric 
    \begin{equation}\label{def:d}
        \varrho((u,v), (u',v')) = \| u - u'\|_\U + \| v - v'\|_\U.
    \end{equation}
    Then there exists an independent constant $C> 0$ so that for any $\epsilon >0$,
    with probability $(1 - \ell_{\widehat{\T}} C \chi_q(\eta) M^{-1/d}/\epsilon) (1 - C \chi_q(\mu) N^{-1/d}/\epsilon)$, it holds that
    \begin{equation*}
    \begin{aligned}
        \W_1(\widehat{\nu}^{N, M}, \nu) 
        \le   &  \Cstab'(y) [1 \vee \Diam(\U) ] \left(  \| \widehat{T}^\dagger - T^\dagger \|_{L^2(\eta)} + \epsilon \right),
    \end{aligned}
    \end{equation*}
    where $\Cstab'(y) >0$ has the expression 
    \begin{equation*}
    \Cstab'(y) = 
        \frac{ g^2(y) \left(  
         \| (1 + \| \cdot \|_\U) f  \|_{L^1(\mu)} \right) 
        \left( 2 \| \varphi(\cdot, \cdot, y)\|_{L^1(\mu)} +  \ell_\varphi(y) \epsilon  \right)^{1/2}  }
        {\| h(\cdot, y) \|_{L^1(\mu)} (\| h(\cdot, y) \|_{L^1(\mu)} - \ell_h(y) \epsilon ) }.
    \end{equation*}
\end{theorem}

\begin{proof}
 Suppose \Cref{prop:prior-error-bound-M-N} holds and 
 consider the event $\W_2(\mu, \widehat{\mu}^{N,M}) < \epsilon$. 
 Then by \Cref{lem:Cstab-control-bounded-case}
 we have that 
 \begin{equation*}
     \W_1(\widehat{\nu}^{N,M}, \nu) 
     \le C(y) \frac{ \| \widehat{T}^\dagger - T^\dagger \|_{L^2(\eta)} + \epsilon}
        {\| h(\cdot, y) \|_{L^1(\mu)} (\| h(\cdot, y) \|_{L^1(\mu)} - \ell_h(y) \epsilon ) }
 \end{equation*}
 where we have 
 \begin{equation*}
     \begin{aligned}
         C(y) & = g^2(y) ( \|f \|_{L^1(\mu)} + \| f \cdot \| \cdot \|_\U \|_{L^1(\mu)} )\\
                        & \times [ 1 \vee \Diam(\U)] 
                \left( \int_{\U \times \U} [ 1 \vee \ell(u, v,y)]^2 
                \cdot [ f(u) \vee f(v) ]^2 \pi^\star(\dd u, \dd v) \right)^{1/2}.         
     \end{aligned}   
 \end{equation*}
 Note that the integral in the second line still depends on 
 $\widehat{\mu}^{N,M}$ due to the optimal coupling 
 $\pi^\star \in \Pi(\mu, \widehat{\mu}^{N,M}) $ and so we need to control 
 the fluctuations of this term. Applying 
 \Cref{lem:wasserstein-expectation-bound} with the function 
 $\varphi(u, v, y) = \big( 1 + \ell(u,v, y)^2 \big) ( f(u)^2 + f(v)^2)$, 
 and thanks to the hypothesis of the theorem, gives 
 \begin{equation*}
 \int_{\U \times \U} [ 1 \vee \ell(u, v,y)]^2 
                \cdot [ f(u) \vee f(v) ]^2 \pi^\star(\dd u, \dd v)  
                \le 2 \int_\U \big( 1 + \ell^2(u, u, y) \big) f^2(u)
                \mu(\dd u) + \ell_\varphi(y) \epsilon.
 \end{equation*}
 Substituting this bound into the expression for $C(y)$ completes the proof.
 \end{proof}

\begin{remark}
    Assuming both $\epsilon$ and $\| \widehat{T}^\dagger - T^\dagger\|_{L^2(\eta)}$ are small and ignoring higher order terms, this 
    theorem simply states that the posterior 
    error is controlled by the approximation error of  $\widehat{T}^\dagger$
    and $\epsilon$ itself which quantifies the stochastic error due to 
    the empirical approximation of the prior. The multiplying constant
    $\Cstab'(y)$ in the bound, although complex in its expression, is effectively
    controlled by generalized moments of the prior $\mu$ and hence can 
    be bounded so long as $\mu$ has sufficiently light tails.
\end{remark}

\begin{remark}
    Note that the reason why we considered $\U$ to be bounded 
    in this section is largely due to the  $\| h(\cdot, y) \|_{L^1(\widehat{\mu})}$ factor in \eqref{eq:Cstab-full-def} since 
    this quantity is random when $\widehat{\mu} = \widehat{\mu}^{N,M}$.
    When $\U$ is bounded any Lipschitz $h$ is automatically globally 
    Lipschitz which allows us to control the variation of  
    $\| h( \cdot, y) \|_{L^1(\mu)} -  \| h( \cdot, y) \|_{L^1(\widehat{\mu}^{N,M})}$. Hence  our proof in this 
    section can naturally be extended to the case of a globally 
    Lipschitz $h$ with unbounded $\U$.
\end{remark}

\begin{remark}
    We note that the dependence of $\Cstab'$ on the data $y$ is 
    explicit through the functions $h$ and $\varphi$ and their 
    Lipschitz constants. This provides a path for further obtain high-probability bounds on posterior perturbations with respect to the 
    data $y$ as well as the empirical data. This requires detailed calculations using the explicit forms of $h$ and $\varphi$, 
    for specific examples which we leave as a future research direction.
\end{remark}

\begin{remark}
We note that our bounds can be simplified whenever the functions $\ell$ and $f$ are bounded. 
For example, suppose that  
\begin{equation*}
\ell(u,v,y)\le \ell_{\max}(y) <+\infty
\qquad \text{and} \qquad 
f(u)\le f_{\max} <+\infty \qquad\text{for all }u,v\in \U.
\end{equation*}
In this case, we can drop the assumptions on the function $\varphi$ in \Cref{thm:posterior-error-bound-bounded-case} and directly bound the constant $\Cstab(y)$ in \eqref{eq:cstab-lemma-display} as 
\begin{equation*}
    \Cstab(y) \le \frac{g^2(y) f_{\max}^2 \| 1 + \| \cdot \|_\U \|_{L^1(\mu)} }{\| h(\cdot, y) \|_{L^1(\mu)} (\| h(\cdot, y) \|_{L^1(\mu)}-\ell_h(y) \epsilon)} [ 1 \vee \Diam(\U)] [1 \vee \ell_{\max}(y)],
\end{equation*}
which allows us to obtain an equivalent version of \Cref{thm:posterior-error-bound-bounded-case} 
with a cleaner constant $\Cstab$ but essentially the same rate in terms of 
the parameter $\epsilon$ and the approximation error $\| \widehat{T}^\dagger - T^\dagger \|_{L^2(\eta)}$.
The same bound can also be extended to the unbounded support case 
described in the next section with the exception that we will assume that $\ell(\cdot, \cdot, y)$ and 
$f$ are only locally bounded over $\U$. 
\end{remark}

\subsubsection{The unbounded support case}
\label{sec:posterior-bounds-unbounded-support}

Let us now turn our attention to the case where 
$\U$ is unbounded and hence $\mu$ may have unbounded support. 
We will deal with this case by truncating the prior 
to a ball of radius $r$ and applying our bounded support 
result with an additional error term due to the 
trimming of the tails. 
To this end,  define the trimmed prior 
\begin{equation*}
    \mu_r(A) := \frac{1}{\mu(B_r)} \mu(A \cap B_r) 
\end{equation*}
for all Borel sets $A \subseteq \U$ and $B_r$ denoting 
the ball of radius $r > 0$ in $\U$. Further define 
the trimmed posterior 
\begin{equation*}
    \frac{\dd \nu_r}{\dd \mu_r} = \frac{1}{Z_r(y)} 
    \exp( - \Phi(u;y) ), \qquad 
    Z_r(y):= \mu_r( \exp(- \Phi(\cdot; y) ) ).
\end{equation*}
We then have the following lemma, quantifying the effect of 
trimming the prior tails. 
\begin{lemma}\label{lem:trimmed-posterior-bound}
    Consider the above setup, we then have the bounds\footnote{Indeed, the fact that $\mu$ and $\nu$ are 
    prior and posterior measures is innocuous in this lemma and 
    the bounds apply to any measure and its corresponding 
    trimmed version. We simply state them this way for easier 
    use later.} 
    \begin{equation*}
        \W_1(\nu, \nu_r) \le \frac{2}{r} \nu( \| \cdot \|_\U)^2, \qquad \W_2(\mu, \mu_r) \le \frac{2}{r} \mu( \| \cdot \|_\U^2).
    \end{equation*}
\end{lemma}
\begin{proof}
    Let us begin with the posterior $\W_1$ bound.
    Define the measure 
    \begin{equation*}
        \nu_r^c(A) := \frac{1}{\nu(B_r^c)} 
        \nu(A \cap B_r^c)
    \end{equation*}
    where $B_r^c$ is the complement of $B_r$ in $\U$.
    Then consider the coupling $\pi_1 \in \Pi(\nu, \nu_r)$
    where $(u, u') \sim \pi_1$ are generated by the 
    following procedure
    \begin{equation*}
    \begin{aligned}
        \text{draw}\quad u \sim \nu \quad 
        \text{if}\quad  \|u\|_\U \le r, \quad \text{set}\quad  u' = u, 
        \quad \text{else}, \quad \text{draw}\quad 
        u' \sim \nu_r.
    \end{aligned}
    \end{equation*}
    Then by the definition of $\W_1$ we have that
    \begin{equation*}
    \begin{aligned}
        \W_1(\nu, \nu_r) &\le 
        \int_{\U \times \U} \| u - u'\|_\U \pi_1(\dd u, \dd u') \\
        & = \int_{\{ u \in B_r\}} 
        \int_\U \| u - u' \|_\U \pi_1(\dd u' \mid u) \pi_1(\dd u) 
        + \int_{\{ u \in B_r^c\}} 
        \int_\U \| u - u' \|_\U \pi_1(\dd u' \mid u) \pi_1(\dd u).
    \end{aligned}
    \end{equation*}
    The first term is zero by construction while for the second term
    we have the bound 
    \begin{equation*}
    \begin{aligned}
        \int_{\{ u \in B_r^c\}} 
        \int_\U \| u - u' \|_\U \pi_1(\dd u' \mid u) \pi_1(\dd u)
        &\le \int_{\{ u \in B_r^c\}} 
        \int_\U \| u \|_\U + \| u' \|_\U \pi_1(\dd u' \mid u) \pi_1(\dd u) \\
        & \le 2 \nu( \U \cap B_r^c) \nu( \| \cdot \|_\U ),
        \end{aligned}
    \end{equation*}
    where in the last inequality we used the 
    fact that $\| u \|_\U \ge \| u' \|_\U$ a.s. under $\pi_1$.
    An application of Markov's inequality 
    yields $\nu( \U \cap B_r^c) \le \frac{\nu( \| \cdot \|_\U)}{r}$
    which completes the proof of the $\W_1$ bound for priors. 

    We can apply the same proof technique to the posteriors, 
    writing $\pi_0 \in \Pi(\mu, \mu_r)$ for the analogous 
    coupling obtained by replacing $\nu, \nu_r$ with $\mu, \mu_r$
    in the definition of $\pi_1$. Then we obtain the bound 
    \begin{equation*}
        \W_2^2( \mu, \mu_r)
        \le 2 \int_{\{ u \in B_r^c\}} 
        \int_\U \| u\|_\U^2 + \| u' \|_\U^2 \pi_0( \dd u' \mid u) 
        \pi_0(\dd u)\le 4 \mu( \U \cap B_r^c) \mu( \| \cdot \|_\U^2),
    \end{equation*}
    where the extra factor of 2 is due to $(a + b)^2 \le 2 (a^2 + b^2)$.  Applying Markov's inequality we have 
    $\mu( \U \cap B_r^c) \le \frac{\mu(\| \cdot \|_\U^2)}{r^2}$
    which completes the proof.
\end{proof}

Equipped with the above lemma we are finally ready to present our complete 
error bound for the case of unbounded parameter spaces $\U$.

\begin{theorem}\label{thm:posterior-error-bound-unbounded-case}
    Suppose $d > 4$, $\U$ is unbounded, and fix $r \ge 1$.
    Consider the trimmed generative prior model 
    \begin{equation*}
        \widehat{\mu}^{N,M}_r := \widehat{T}^{N,M}_r \# \eta \qquad 
        \widehat{T}^{N,M}_r := \argmin_{T \in \widehat{\T}}
        \W_1( T \# \eta^M, \mu_r^N).
    \end{equation*}
    where the approximation class $\widehat{\T}$ is globally 
    $\ell_{\widehat{\T}}$-Lipschitz. Further write $\widehat{\nu}^{N,M}_r$
    for the posterior measure arising from $\widehat{\mu}^{N,M}_r$. 
    Write $T^\dagger_r$ for the map that satisfies 
    $\mu_r = T^\dagger_r \# \eta$ and define 
    $\widehat{T}^\dagger_r := \argmin_{T \in \widehat{\T}} 
    \| T - T^\dagger_r \|_{L^2(\eta)}$, $\widehat{T}_r^\star := \argmax_{T\in \widehat{\T}} \W_2(T\#\eta, T\#\eta^M)$. 
    
    Suppose the following conditions hold: 
    \begin{enumerate}[label=(\roman*)]
        \item     $\Phi \in \F(f, g, h, \ell, \mu)$ such that 
    $(1 + \| \cdot \|_\U) f \in L^1(\mu_r)$.
    \item The function $h(\cdot, y)$ is $\ell_{h,r}(y)$-Lipschitz 
    over $\U \cap B_r$.
    \item The function $\varphi(\cdot, \cdot, y)$ defined in \eqref{def:varphi} 
    is $\ell_{\varphi, r}(y)$-Lipschitz with respect to the 
    metric $\varrho$ defined in \eqref{def:d} over $(\U \cap B_r) \times (\U \cap B_r)$.
        
        \item $\chi_q(\mu), \chi_q(\eta) <
        + \infty$ for some $q > 2d/(d-2)$.
    \end{enumerate}
    Then there exists a constant $C> 0$ so that for any 
    $\epsilon >0$,  with probability $(1 - \ell_{\widehat{\T}} C \chi_q(\eta) M^{-1/d}/\epsilon) (1 - C \chi_q(\mu) N^{-1/d}/\epsilon)$, it holds that
    \begin{equation*}
        \W_1(\widehat{\nu}^{N,M}_r, \nu) 
        \le \Cstab(y,r) \cdot ( \| \widehat{T}^\dagger_r 
        - T^\dagger \|_{L^2(\eta)} + \epsilon) 
        + \frac{2}{r} \nu( \| \cdot \|_\U)^2, 
    \end{equation*}
    where $\Cstab'(y,r)$ has the expression
    \begin{equation*}
            \Cstab(y,r) = 
        \frac{2 g^2(y) \left(  
         \| (1 + \| \cdot \|_\U) f  \|_{L^1(\mu_r)} \right) 
        \left( 2 \| \varphi(\cdot, \cdot, y)\|_{L^1(\mu_r)} +  \ell_{\varphi,r}(y) \epsilon  \right)^{1/2}  r}
        {\| h(\cdot, y) \|_{L^1(\mu_r)} (\| h(\cdot, y) \|_{L^1(\mu_r)} - \ell_{h,r}(y) \epsilon ) }.
    \end{equation*}
\end{theorem}

\begin{proof}
By the triangle inequality we have 
$\W_1(\widehat{\nu}^{N,M}_r, \nu) \le \W_1(\widehat{\nu}^{N,M}_r, \nu_r) 
+ \W_1(\nu, \nu_r)$. We can bound the second term using 
\Cref{lem:trimmed-posterior-bound} while the first term is
bounded using \Cref{thm:posterior-error-bound-bounded-case} by replacing 
$\mu$ with $\mu_r$ and observing $\Diam(\U) \le 2 r$.
\end{proof}

\begin{remark}
    Observe that the posterior moment $\nu(\| \cdot \|_\U)$ can further be 
    controlled by the prior moment $\mu(\| \cdot \|_\U)$ using 
    our assumptions on the likelihood potential  $\Phi$, i.e., by 
    Cauchy-Schwarz we have
    \begin{equation*}
        \nu( \| \cdot \|_\U) \le \frac{1}{Z(y)} \mu(\| \cdot \|_\U^2)^{1/2}
        (g(y) + \| f \|_{L^2(\mu)}).
    \end{equation*}
    Since $Z(y) \ge \| h(\cdot, y) \|_{L^1(\mu)}$ we can further obtain the 
    bound 
    \begin{equation*}
        \nu( \| \cdot \|_\U)^2 \le \frac{ (g(y) + \| f \|_{L^2(\mu)})^2 \mu( \| \cdot \|_\U^2)}{\| h(\cdot, y) \|_{L^1(\mu)}^2}. 
    \end{equation*}
    This reveals that the posterior error due to trimming, has the 
    same detrimental scaling due to the evidence $Z(y)$ or similarly 
    $\| h(\cdot, y) \|_{L^1(\mu)}$ as the constant $\Cstab$ and so 
    the entire bound becomes innocuous when the data $y$ is 
    "unlikely" in the sense of having very small evidence $Z(y)$.
    This is a well-known phenomenon that was observed in \cite{sprungk2020local} and highlights the mechanism underpinning the 
    brittleness of BIPs in \cite{owhadi2015brittleness}.
\end{remark}

\section{Numerical Results}\label{sec:Numerical Results}
In this section we collect a series of numerical experiments 
aimed at verifying some of our error bounds. In particular, we 
focus on the posterior perturbation bounds in terms 
of the priors from \Cref{sec:wasserstein-perturbation}. 
In \Cref{sec:Toy Example} we present a series of 2D benchmarks 
where we explicitly check that the posterior $\W_1$ distance is 
controlled by the prior $\W_2$ distance. In \Cref{sec:HD Example} 
we present an example of a nonlinear PDE inverse problem with a 
generative prior and demonstrate the effectiveness of generative 
priors for more complex and realistic examples. Our code to reproduce the numerical results can be found in a public Github repository\footnote{https://github.com/TADSGroup/DataDrivenBIPwithGenerativePrior/tree/main}.

\subsection{2D Benchmarks}\label{sec:Toy Example}
We begin with a set of low dimensional benchmarks where posterior 
samples can be generated accurately and cheaply using importance sampling. 
This allows us to compute precise Wasserstein distances 
between true posteriors and those with a generative prior without 
additional noise due to posterior inference algorithms such 
as MCMC that can often pollute our results. 

\paragraph{Problem setup}
We took $\U = \R^2$ and took our prior measures $\mu$ to be 
the benchmark distributions taken from \cite{kingma2018glow}; 
these are depicted in 
\Cref{Fig_2DResult}. 
Here we present the Swissroll, Pinwheel, 
and Checkerboard distributions since we found these  
to be good representatives of our findings. We found 
very similar results for the other benchmarks from 
\cite{kingma2018glow}.

To define the posterior measures $\nu$ we considered the likelihood 
potential 
\begin{equation*}
    \Phi(u;y) = \frac{1}{2\sigma^2} \| y - F u \|_2^2, 
    \qquad F = \begin{bmatrix}
        1 & 0 \\ 
        0 & 0 
    \end{bmatrix}.
\end{equation*}
This likelihood arises from the familiar data model $y = Fu + \xi$
where the additive noise $\xi \sim N(0, \sigma^2 I)$. Here we chose $\sigma = 0.5$. The resulting
posteriors are depicted in \Cref{Fig_2DResult}
for a fixed choice of the data $y = ( 0 , 0)^T$. 
The presented heatmaps were obtained by 
generating prior samples that were reweighted
and resampled according to their 
likelihoods. We took these to be samples from 
the true posterior measures $\nu$ in our experiments.

\paragraph{The generative prior}
We used a Wasserstein GAN with gradient penalty
(WGAN-gp) \cite{gulrajani2017improved} as a generative approximation to the priors above and computing the maps $\widehat{T}$
and in turn the priors $\widehat{\mu}$. 
In general, we found the WGAN-gps to be difficult to train in order 
to obtain highly accurate approximations to the priors, to this end, 
we modified the training procedure by borrowing ideas 
from stochastic interpolants \cite{albergo2025stochastic}; the 
details are summarized in \Cref{app:WGAN-training}.
With a generative prior at hand, we can use the same 
reweighting trick as the true posterior to obtain samples 
from the approximate posteriors $\widehat{\nu}$
In \Cref{Fig_2DResult} we show an instance of 
the learned priors and the resulting posteriors.

\paragraph{Experiments and results}
Our goal was to investigate the validity of the perturbation 
bounds of \Cref{sec:wasserstein-perturbation} and in particular
\Cref{thm:main-1-wasserstein-stability}. To this end, we compared 
the $\W_1(\nu, \widehat{\nu})$ and compared it with 
$\W_2(\mu, \widehat{\mu})$ for various 
generative priors $\widehat{\mu}$ by varying the size of 
the training data for the GAN as well as the size of the 
networks. In particular, we considered the following scenarios: 
\begin{itemize}
    \item {\it Effect of the sample size $N$:} We fixed our WGAN-gp
    architecture and training procedure (see \Cref{app:WGAN-training}) and increased the number of prior training samples $N$
    from $2^9$ to $2^{14}$. 

    \item {\it Effect of network width:} We fixed the 
    number of training samples $N = 10000$ and used a generator 
    with three layers but we modified the width of the 
    layers from $2^4$ to $2^8$. The number training epochs 
    and the discriminator architectures were kept fixed. 

    \item {\it Effect of training epochs:}
    Finally, we considered a fixed $N = 10000$ and a network 
    of depth 3 and width 128 and increased the number of 
    training epochs from $2^8$ to $2^{12}$. Once again the 
    discriminator architectures were kept fixed.    
\end{itemize}
We think of the first experiment (modifying $N$) as 
representing the stochastic training errors while the 
other experiments reflect the approximation bias/error 
of our parameterizations. 

We present our results in \Cref{WDsample complexity,WDwidth,WDepochs}
where the prior $\W_2$ distances are compared with the posterior 
$\W_1$ distances as a function of the underlying variables. 
We repeated each experiments 5 times and presented the average 
errors along with the error bars obtained from repeated experiments. 
We consistently observed that the posterior $\W_1$ distance 
was controlled by the prior $\W_2$ distances with the underlying 
slopes matching closely as shown in \Cref{tab:WDslopes}. 
We take this as a validation of our bound in \Cref{thm:main-1-wasserstein-stability}. Interestingly, we observed that the 
prior $\W_2$ slopes (and by extension the posterior slops) 
did not match the empirical $N^{-1/2}$ rate predicted by 
\cite{fournier2015rate}. We take this as evidence that 
WGAN-gp is not a good estimator of the prior in the $\W_2$ sense.

\begin{figure}[htp]
\begin{overpic}[width=0.23\textwidth]{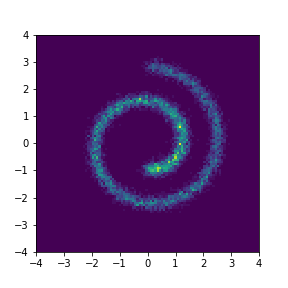}%
\put(250,920){\scriptsize True prior $\mu$}
\put(-30,320){\scriptsize \rotatebox{90}{Swissroll}}
\end{overpic}%
\hspace{0.01\textwidth}%
\begin{overpic}[width=0.23\textwidth]
{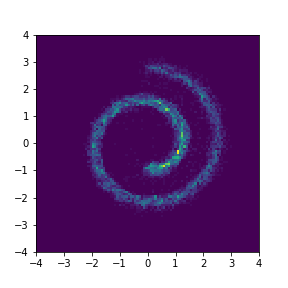}%
\put(100,920){\scriptsize Approximate prior $\widehat{\mu}$}
\end{overpic}
\hspace{0.01\textwidth}%
\begin{overpic}[width=0.23\textwidth]
{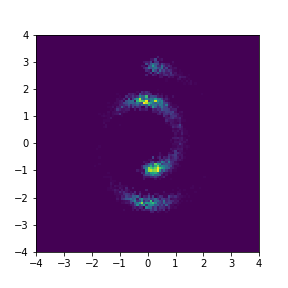}%
\put(180,920){\scriptsize True posterior ${\nu}$}
\end{overpic}
\hspace{0.01\textwidth}%
\begin{overpic}[width=0.23\textwidth]
{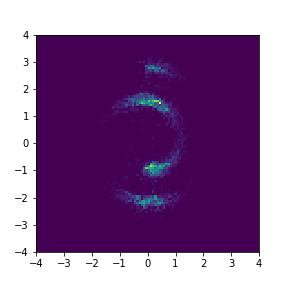}%
\put(30,920){\scriptsize Approximate posterior $\widehat{\nu}$}
\end{overpic}
\\
\begin{overpic}[width=0.23\textwidth]{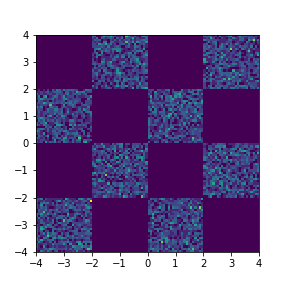}%
\put(-30,250){\scriptsize \rotatebox{90}{Checkerboard}}
\end{overpic}%
\hspace{0.01\textwidth}%
\begin{overpic}[width=0.23\textwidth]
{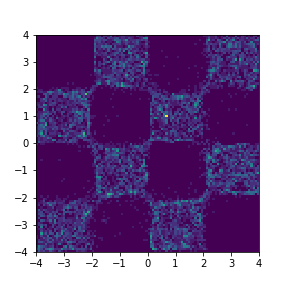}%
\end{overpic}
\hspace{0.01\textwidth}%
\begin{overpic}[width=0.23\textwidth]
{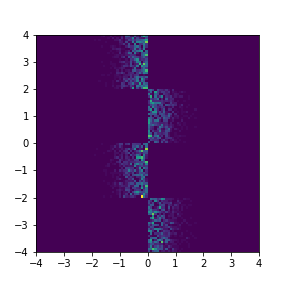}%
\end{overpic}
\hspace{0.01\textwidth}%
\begin{overpic}[width=0.23\textwidth]
{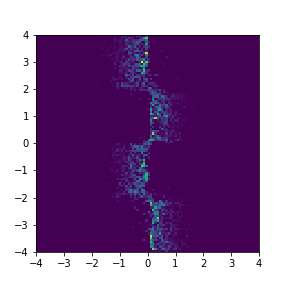}%
\end{overpic}
\\
\begin{overpic}[width=0.23\textwidth]{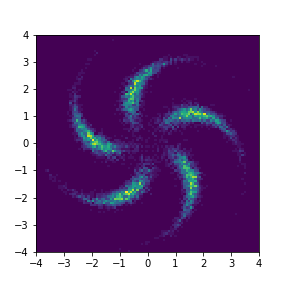}%
\put(-30,320){\scriptsize \rotatebox{90}{Pinwheel}}
\end{overpic}%
\hspace{0.01\textwidth}%
\begin{overpic}[width=0.23\textwidth]
{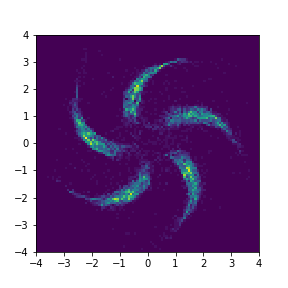}%
\end{overpic}
\hspace{0.01\textwidth}%
\begin{overpic}[width=0.23\textwidth]
{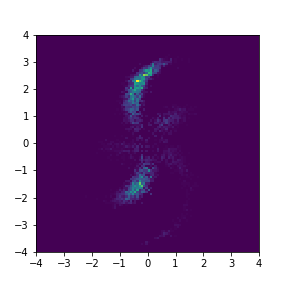}%
\end{overpic}
\hspace{0.01\textwidth}%
\begin{overpic}[width=0.23\textwidth]
{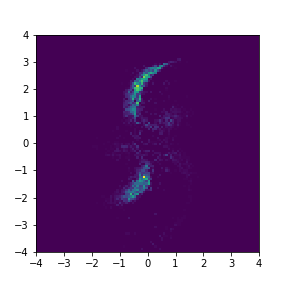}%
\end{overpic}
\caption{Example of 2D Benchmarks from \Cref{sec:Toy Example} with three example distributions. From 
left to right, the first column depicts the true prior $\mu$, the second column shows the data-driven prior
$\widehat{\mu}$, 
the third column shows the true posterior samples $\nu$, while the fourth column shows samples from the 
data-driven posterior $\widehat{\nu}$. The heatmaps for the posterior were generated 
by reweighting the prior samples using the likelihood.
}\label{Fig_2DResult}
\end{figure}


\begin{figure}[htp]
\begin{overpic}[width=0.3\textwidth]{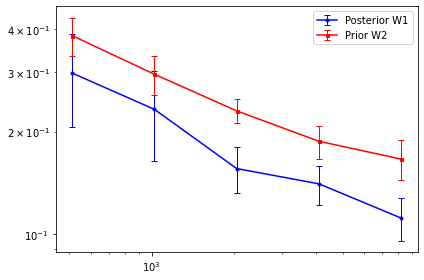}%
\put(400,670){\scriptsize Swissroll}
\put(250,-50){\scriptsize Training sample size}

\put(-30,50){\scriptsize \rotatebox{90}{Wasserstein distance}}
\end{overpic}%
\hspace{0.03\textwidth}%
\begin{overpic}[width=0.3\textwidth]{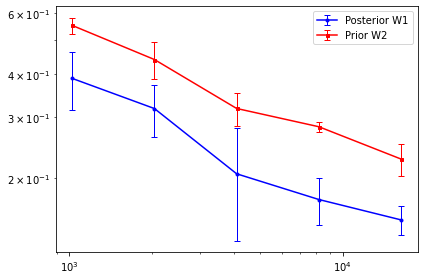}%
\put(350,670){\scriptsize Checkerboard}

\put(250,-50){\scriptsize Training sample size}

\end{overpic}%
\hspace{0.03\textwidth}%
\begin{overpic}[width=0.3\textwidth]{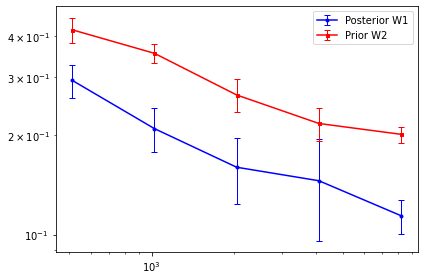}%
\put(400,670){\scriptsize Pinwheel}

\put(250,-50){\scriptsize Training sample size}

\end{overpic}%
\caption{
Comparing the convergence of data-driven priors in the $\W_2$ metric versus the subsequent 
posteriors in the $\W_1$ metric as a function of the training sample size for 
the 2D benchmarks from \Cref{sec:Toy Example}. In all cases the prior $\W_2$ distance controls 
the posterior $\W_1$ distance.  Slopes of linear fits to these curves, indicating 
rate of convergence with the sample size, are reported in 
\Cref{tab:WDslopes}.
}\label{WDsample complexity}
\end{figure}

\begin{figure}
\begin{overpic}[width=0.3\textwidth]{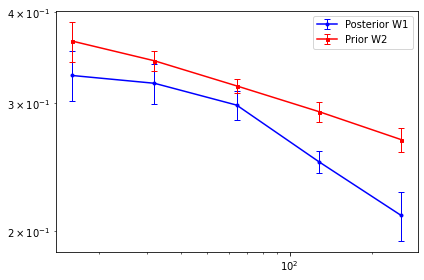}%
\put(420,670){\scriptsize Swissroll}
\put(350,-50){\scriptsize Width of GAN}

\put(-30,50){\scriptsize \rotatebox{90}{Wasserstein distance}}
\end{overpic}%
\hspace{0.03\textwidth}%
\begin{overpic}[width=0.3\textwidth]{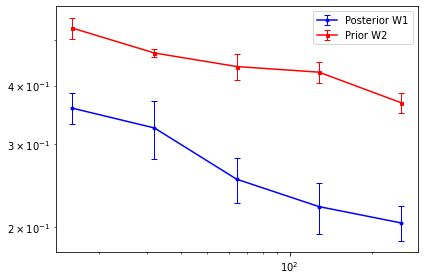}%
\put(370,670){\scriptsize Checkerboard}
\put(350,-50){\scriptsize Width of GAN}

\end{overpic}%
\hspace{0.03\textwidth}%
\begin{overpic}[width=0.3\textwidth]{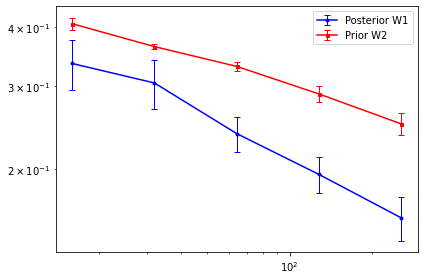}%
\put(430,670){\scriptsize Pinwheel}
\put(350,-50){\scriptsize Width of GAN}

\end{overpic}%
\caption{
Comparing the convergence of data-driven priors in the $\W_2$ metric versus the subsequent 
posteriors in the $\W_1$ metric as a function of the width of the generator for 
the 2D benchmarks from \Cref{sec:Toy Example}. In all cases the prior $\W_2$ distance controls 
the posterior $\W_1$ distance.  Slopes of linear fits to these curves, indicating 
rate of convergence with width of the generator, are reported in 
\Cref{tab:WDslopes}.
}\label{WDwidth}
\end{figure}

\begin{figure}
\begin{overpic}[width=0.3\textwidth]{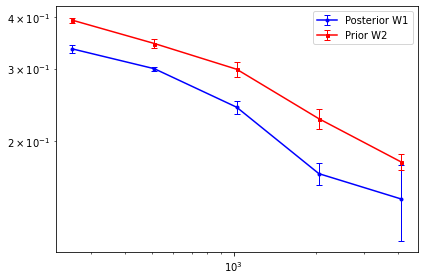}%
\put(420,670){\scriptsize Swissroll}
\put(200,-50){\scriptsize Training Epochs of GAN}
\put(-30,50){\scriptsize \rotatebox{90}{Wasserstein distance}}
\end{overpic}%
\hspace{0.03\textwidth}%
\begin{overpic}[width=0.3\textwidth]{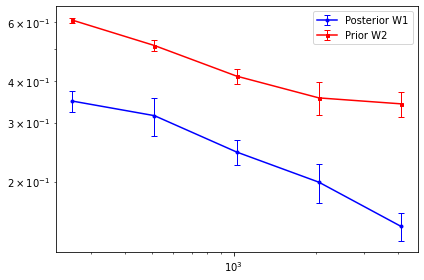}%
\put(370,670){\scriptsize Checkerboard}
\put(200,-50){\scriptsize Training Epochs of GAN}
\end{overpic}%
\hspace{0.03\textwidth}%
\begin{overpic}[width=0.3\textwidth]{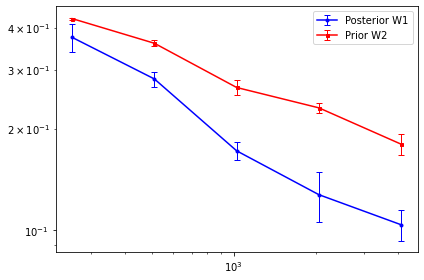}%
\put(430,670){\scriptsize Pinwheel}
\put(200,-50){\scriptsize Training Epochs of GAN}
\end{overpic}%
\caption{
Comparing the convergence of data-driven priors in the $\W_2$ metric versus the subsequent 
posteriors in the $\W_1$ metric as a function of the number of training epochs for 
the 2D benchmarks from \Cref{sec:Toy Example}. In all cases the prior $\W_2$ distance controls 
the posterior $\W_1$ distance.  Slopes of linear fits to these curves, indicating 
rate of convergence with the number of training epochs, are reported in 
\Cref{tab:WDslopes}.
}\label{WDepochs}
\end{figure}

\begin{table}[htp]
    \centering
    \begin{tabular}{ |l |c | c| c| c | c| c|  }
        \hline
        & \multicolumn{2}{|c|}{Training sample size} & \multicolumn{2}{|c|}{Network width} 
        & \multicolumn{2}{|c|}{Training epochs} \\ \hline 
            & Prior $\W_2$ & Post. $\W_1$ & Prior $\W_2$ & Post. $\W_1$ & Prior $\W_2$ & Post. $\W_1$ \\ \hline
        Swissroll & -0,307 &-0.357 &-0.113 &-0.164 & -0.289 & -0.327 \\ \hline
        Checkerboard & -0.321& -0.359& -0.119& -0.218& -0.219& -0.316 \\ \hline
        Pinwheel & -0.281& -0.325& -0.174& -0.282& -0.312& -0.485\\ \hline 
    \end{tabular}
    \caption{We report the slopes of linear fits, in the log-log scale, 
    of the $W_1$ distance of data-driven posteriors compared to 
    the ground truth alognside the $W_2$ distance of data-driven priors. In both cases we report the 
    errors as a function of the sample size of the training data, width of WGAN-gp, and number of training epochs. 
     The left two  columns  correspond to the plots in \Cref{WDsample complexity}, the third and fourth columns correspond to the plots in \Cref{WDwidth}, while the two columns on the right correspond to \Cref{WDepochs}.}
    \label{tab:WDslopes}
\end{table}

\subsection{A PDE Example}\label{sec:HD Example}
We now consider a PDE inverse problem as an instance of a high-dimensional problem, in order to demonstrate 
the effectiveness of generative priors for solving inverse problems in scientific computing. The problem 
considered here is synthetic and designed specifically to demonstrate the efficacy of data-driven 
BIPs in handling complex and multi-modal priors and posteriors arising from PDEs and using simple 
sampling algorithms. We also note that while our analysis in the rest of the paper was tailored to 
finite dimensional parameters, i.e., we always assumed $\U \subseteq \RR^d$, in this example 
the parameter is truly infinite dimensional although we work with a discretization of the problem.
However, we use dimension reduction techniques alongside function space MCMC algorithms 
to ensure the resulting methodology is mesh independent. 

\paragraph{Problem Setup} 
We consider the Darcy flow PDE which is a simple model of the flow of a fluid in a porous medium: 
\begin{equation*}
\begin{cases}
- \text{div} (\exp(u(x)) \nabla p(x)) = f(x), \quad x \in (0,1)^2, & x \in (0,1)^2, \\
p(x) = 0, & x \in \partial(0,1)^2.
\end{cases}    
\end{equation*}
where $u$ is the log of the permeability field and $p$ denotes the pressure. The inverse problem of interest 
is that of estimating the permeability field $\exp(u)$ from pointwise measurements of the 
pressure $p$. To this end, we define the likelihood 
\begin{equation*}
    \Phi(u;y) = \frac{1}{2\sigma^2} \| y - F(u) \|_2^2, 
\end{equation*}
where $F: u \rightarrow (p(x_1) \dots p(x_{300}) \in \RR^{300}$ with the $x_i$ denoting a fixed 
collection of observation points in the domain; see \Cref{HD01Result}. Hence, the forward map $F$ 
consists of the composition of the parameter-to-solution map of the PDE  with a pointwise evaluation map.
We approximate this map by solving the PDE using finite elements.  
To make the problem interesting we assume that the prior on $u$ is the distribution of the MNIST data set. 
Since the map $F$ is smoothing and loses a lot of information about $u$ we expect that a generative 
model trained on MNIST should give rise to a multi-modal posterior since many choices of $u$ can 
generate the observed data.

\paragraph{The generative prior}
We discretized $\U = \R^{784}$ viewing MNIST images as 
piecewise constant functions\footnote{We can also work with a 
PCA dimension reduction to ensure that the formulation is discretization invariant.} and 
trained a WGAN-gp on the MNIST data set as our generative prior; see \Cref{app: HDWGANgp} 
for details. 


\paragraph{Posterior sampling in the latent space}
Since our problem is high dimensional we can no longer use the likelihood to re-weight the prior samples 
as in the 2D benchmarks in \Cref{sec:Toy Example}. Instead, we employ the preconditioned-Crank-Nicolson (pCN) 
algorithm of \cite{cotter2013mcmc} in the latent space of the GAN. To be precise, let $\widehat{T}: z \mapsto u$
denote our GAN, transforming i.i.d. Gaussian noise $\eta$ 
to the MNIST image space. We then consider the inverse problem 
\begin{equation*}
    \frac{\dd \widehat{\gamma}}{\dd \eta}(z) = \frac{1}{Z(y)} \exp \left( - \Phi( \widehat{T}(z) ; y) \right), \qquad 
    Z(y) = \eta( \exp( - \Phi(T(z) ; y) ) ),
\end{equation*}
where $\gamma$ now denotes the posterior on the latent variable $z$ and we slightly abused notation to 
indicate the normalizing constant as $Z(y)$ as before. The data-driven posterior is simply 
given by $\widehat{T} \# \widehat{\gamma}$. In other words, if we use pCN to generate posterior samples 
in the latent space we can then push those samples through the GAN to obtain samples from the 
target posterior $\widehat{\nu}$ supported on the image space. We note that the articles \cite{patel2022solution,bohra2023bayesian} used a similar approach for posterior sampling. 
The resulting procedure is summarized in \Cref{DarcyAlg}.

\paragraph{Experiments and results} 
We considered two experiments using the same set of observation points with the true parameter $u$ 
chosen as an image of the digit "3" as shown in \Cref{HD01Result}(a). We considered two choices of $\sigma$ arising from $10\%$ and $20\%$ noise-to-signal ratio. 

\Cref{HD01Result} summarizes our results in the high noise regime (20\% noise-to-signal ratio) using $2 \times 10^5$ MCMC samples. We observed 
that the posterior mean did not match the correct digit and the posterior variance was high. Indeed, the posterior 
is multimodal as shown in \Cref{HD01images} where we can see that the digits "3, 8, 2, 5" appear in 
posterior samples. This is precisely the behavior we aimed to capture with our method since 
simple MCMC algorithms such as pCN often struggle to traverse the support of multi-model posteriors. 
However, using pCN in the latent space of the GAN appears to resolve this issue in our example. 
\Cref{HD005Result}(d) and (e) show coordinate-wise autocorrelation functions and effective sample sizes
of pCN, indicating good performance and mixing of the algorithm.

\Cref{HD005Result} shows the results of our experiment in the low noise regime with noise-to-signal ration 
of $10\%$ using $10^6$ MCMC samples. As expected, the MCMC chain converges slower in the low noise regime 
due to the concentration of the prior.
We show the same quantities as in the previous experiment. We see that the posterior mean is now 
a close match to the original image and appears to be unimodal based on the independent samples 
shown in \Cref{HD005images}. We believe that in this instance the observed data is of high enough 
quality that allows us to recover a close approximation to the true image.

\begin{algorithm}[htpb!]
\caption{Posterior sampling procedure for the Darcy flow PDE inverse problem}\label{DarcyAlg}
\begin{algorithmic}[1]
\STATE \textbf{Input:}
\STATE Data $y$, noise variance $\sigma$, MNIST generator $\widehat{T}$ with reference measure 
$\eta = \mathcal{N}(0, I)$, 
and discrete parameter to observation map $F$, and pCN step size $\beta \in (0,1)$.
\STATE \textbf{Initialize MCMC:}
\STATE Draw $z^{(0)} \sim \eta$, evaluate the generator $u^{(0)} = \widehat{T}(z^{(0)})$, and 
compute $\Phi(u^{(0)}; y)$
\FOR{$i = 1$ to $N_{\text{sample size}}$}
    \STATE Propose: $z' = \beta z^{(i-1)} + \sqrt{1 - \beta^2}  \xi$ where $\xi \sim \eta$
    \STATE Evaluate the generator: $u' = \widehat{T}(z')$
    \STATE Compute the likelihood $\Phi( u' ; y) $ by evaluating discrete PDE solver
    \STATE Compute acceptance probability: $\alpha = \min\left(1, \exp\left(\Phi(u^{(i-1)};y) - \Phi(u';y)\right)\right)$
    \STATE With probability $\alpha$ accept $z'$ and set $z^{(i)} = z'$ and collect $u^{(i)} = u'$
    \STATE Otherwise reject and set $z^{(i)} = z^{(i-1)}$ and collect $u^{(i)} = u^{(i-1)}$
\ENDFOR
\end{algorithmic}
\end{algorithm}


\begin{figure}[htp]
\centering
\setlength{\tabcolsep}{4pt} 
\renewcommand{\arraystretch}{0} 
\begin{tabular}{ccc}
\begin{overpic}[width=0.22\textwidth]{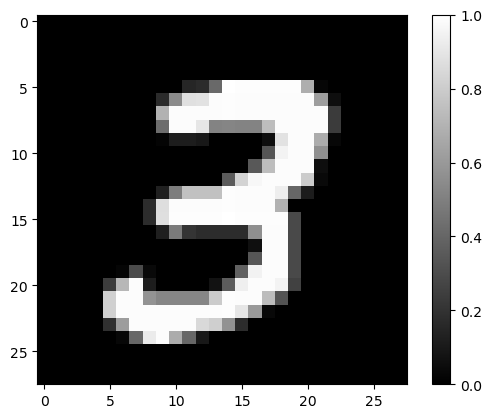}
\put(450,-100){\centering (a)}
\end{overpic} &
\begin{overpic}[width=0.22\textwidth]{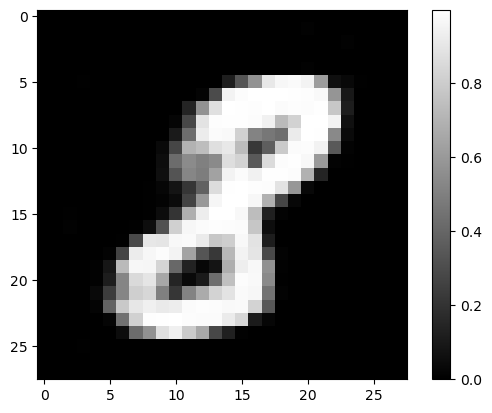}
\put(450,-100){\centering (b)}
\end{overpic} &
\begin{overpic}[width=0.22\textwidth]{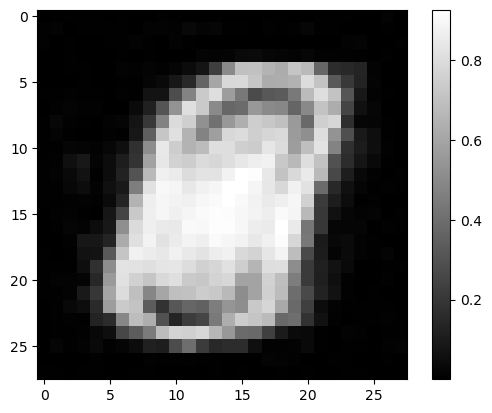}
\put(450,-100){\centering (c)}
\end{overpic}
\\[20pt]
\begin{overpic}[width=0.28\textwidth, height=3cm]{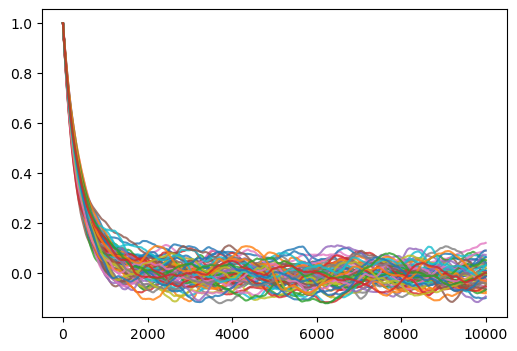}
\put(470,-50){\scriptsize Lag}
\put(-50,300){\scriptsize \rotatebox{90}{ACF}}
\put(470,-150){\centering (d)}
\end{overpic} &
\begin{overpic}[width=0.28\textwidth, height=3cm]{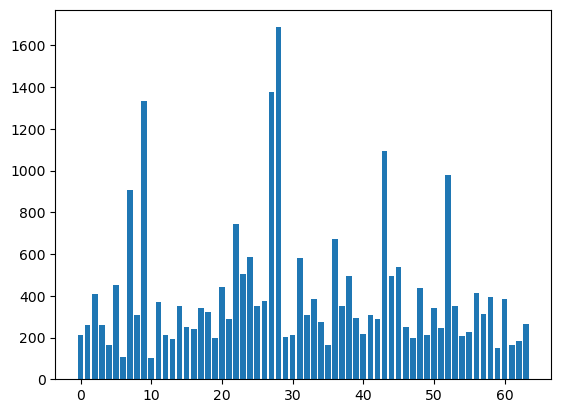}
\put(400,-50){\scriptsize Dimension}
\put(-50,330){\scriptsize \rotatebox{90}{ESS}}
\put(470,-150){\centering (e)}
\end{overpic} &
\begin{overpic}[width=0.22\textwidth, height=3cm]{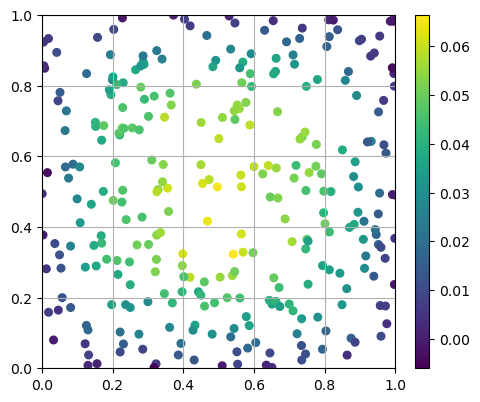}
\put(450,-200){\centering (f)}
\end{overpic} 
\\[20pt]
\end{tabular}
\caption{Results of the Darcy flow experiments with $20\%$ noise-to-signal ratio: (a) the ground truth 
parameter; (b) posterior mean in the image space; (c) pointwise posterior standard deviation in the image 
space; (d) coordinate-wise autocorrelation functions (ACF) of the MCMC chain in the latent space; (e) coordinate-wise effective sample size (ESS) of the MCMC chain in the latent space; 
(f) the data $y$ plotted as the noisy pointwise observations of the pressure field $p$
arising from the true field in panel (a).}\label{HD01Result}
\end{figure}
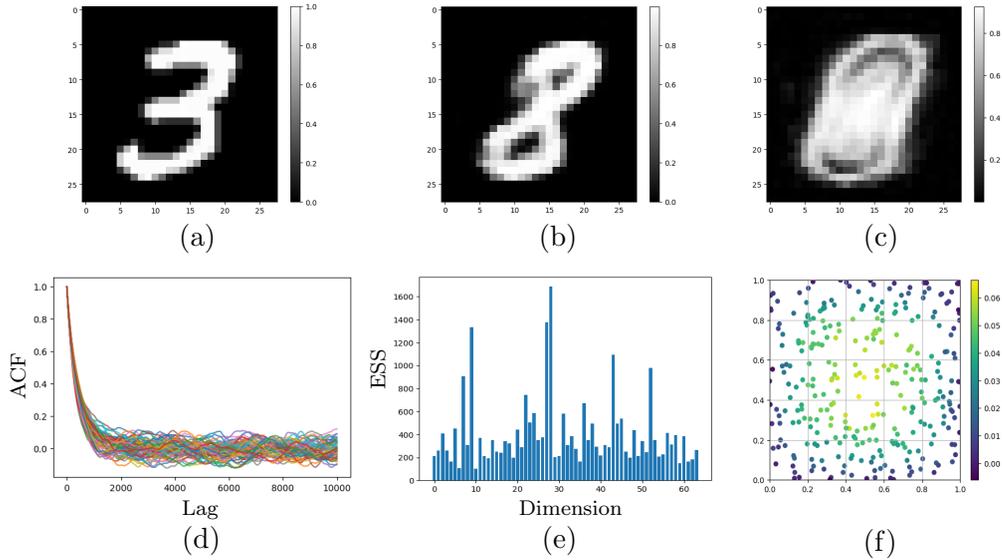

\begin{figure}[htp]
\centering
\setlength{\tabcolsep}{4pt} 
\renewcommand{\arraystretch}{0} 
\begin{tabular}{ccc}
\begin{overpic}[width=0.9\textwidth]{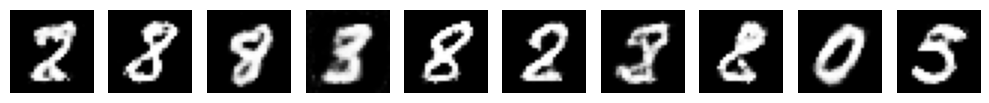}
\end{overpic} &
\\[20pt]
\end{tabular}
\caption{Ten independent posterior samples taken from the MCMC chain (every $10^5$ steps) for our 
experiments with $20\%$ noise-to-signal ratio. The samples show multiple digits, indicating the 
multi-modality of the posterior in the image space.}\label{HD01images}
\end{figure}

\begin{figure}[htp]
\centering
\setlength{\tabcolsep}{4pt} 
\renewcommand{\arraystretch}{0} 
\begin{tabular}{ccc}
\begin{overpic}[width=0.22\textwidth]{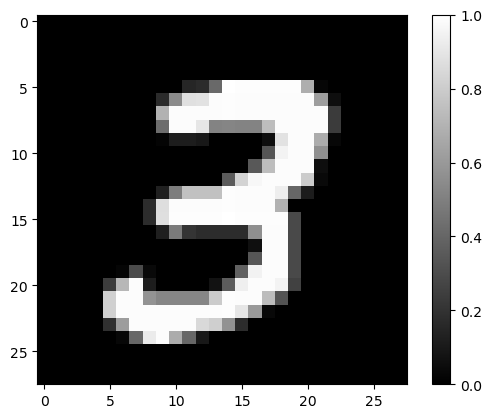}
\put(450,-100){\centering (a)}
\end{overpic} &
\begin{overpic}[width=0.22\textwidth]{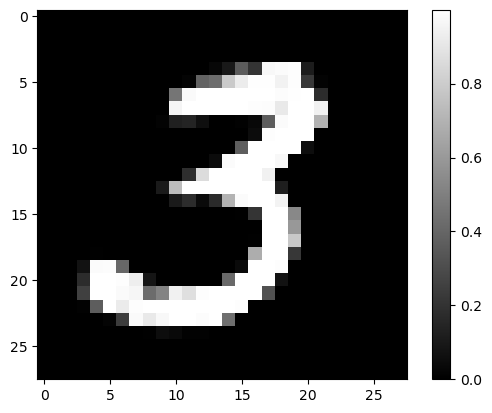}
\put(450,-100){\centering (b)}
\end{overpic} &
\begin{overpic}[width=0.22\textwidth]{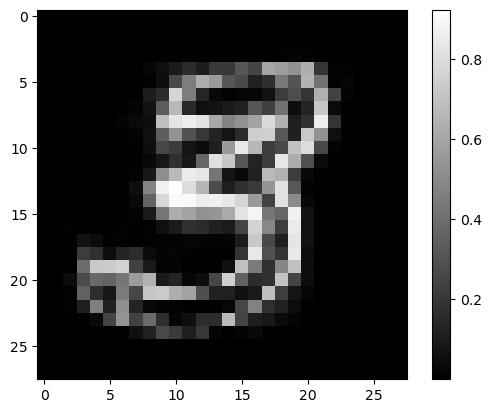}
\put(450,-100){\centering (c)}
\end{overpic}
\\[20pt]
\begin{overpic}[width=0.28\textwidth, height=3cm]{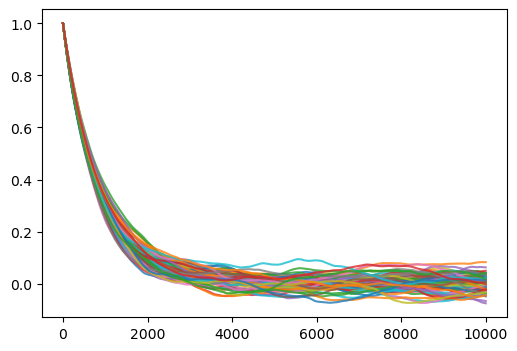}
\put(470,-50){\scriptsize Lag}
\put(-50,300){\scriptsize \rotatebox{90}{ACF}}
\put(470,-150){\centering (d)}
\end{overpic} &
\begin{overpic}[width=0.28\textwidth, height=3cm]{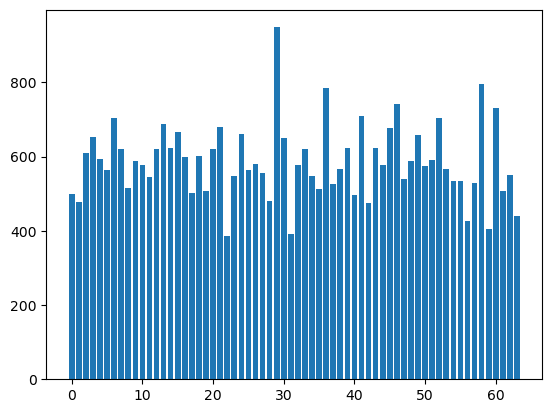}
\put(400,-50){\scriptsize Dimension}
\put(-50,330){\scriptsize \rotatebox{90}{ESS}}
\put(470,-150){\centering (e)}
\end{overpic} &
\begin{overpic}[width=0.22\textwidth, height=3cm]{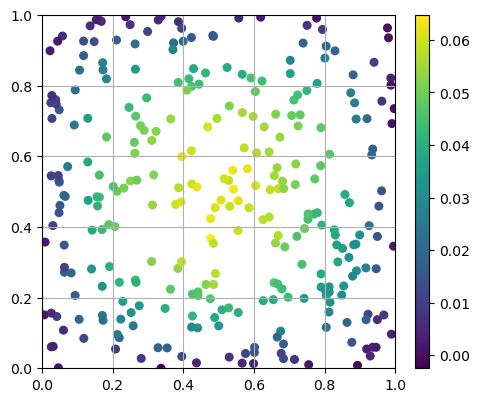}
\put(450,-200){\centering (f)}
\end{overpic} 
\\[20pt]
\end{tabular}
\caption{
Results of the Darcy flow experiments with $10\%$ noise-to-signal ratio: (a) the ground truth 
parameter; (b) posterior mean in the image space; (c) pointwise posterior standard deviation in the image 
space; (d) coordinate-wise autocorrelation functions (ACF) of the MCMC chain in the latent space; (e) coordinate-wise effective sample size (ESS) of the MCMC chain in the latent space; 
(f) the data $y$ plotted as the noisy pointwise observations of the pressure field $p$
arising from the true field in panel (a).
}\label{HD005Result}
\end{figure}

\begin{figure}[htp]
\centering
\setlength{\tabcolsep}{4pt} 
\renewcommand{\arraystretch}{0} 
\begin{tabular}{ccc}
\begin{overpic}[width=0.9\textwidth]{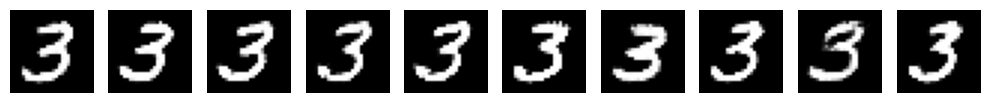}
\end{overpic} &
\\[20pt]
\end{tabular}
\caption{
Ten independent posterior samples taken from the MCMC chain (every $10^5$ steps) for our 
experiments with $10\%$ noise-to-signal ratio. The samples show small variations of the digit "3", indicating the 
multi-modality of the posterior in the image space.
}\label{HD005images}
\end{figure}

\section{Conclusion}\label{sec:Conclusion}
We developed a theoretical analysis of data-driven BIPs with generative priors learned from an 
empirical data. We considered a transport model for our generator and combined ideas from 
the perturbation analysis of BIPs with novel error analysis for minimum Wasserstein generative models 
to obtain quantitative error bounds for data-driven posteriors with respect to a ground truth posterior 
defined in the large data limit of the data-driven prior. Our analysis revealed a simple 
bias-variance trade-off, with an approximation bias coming from our parameterization of 
the generator and a stochastic error coming from the finite training data. However, the details of the 
analysis require various technical assumptions on the smoothness of the likelihood functions as well as 
tail properties of the prior. We further presented numerical experiments that verified the validity of 
some of our bounds in low-dimensional benchmarks where exact sampling of highly multi-modal posteriors was possible.

Our results lead to multiple follow up questions: (i) we primarily considered finite dimensional parameter spaces but most interesting inverse problems in scientific computing concern infinite-dimensional
parameters and many of our results, most notably the empirical convergence rates 
for Wasserstein distance do not readily extend to this setting; 
(ii) most of the constants in our bounds are dependent on the data $y$ and they will deteriorate 
for low-likelihood data, it would be interesting to extend these to high probability bounds or 
bounds in expectation with respect to the data; 
(iii) the minimum Wasserstein analysis is removed from implementation of many generative models since 
the resulting optimization problem requires differentiation through an OT solver. It is interesting 
to extend our analysis to metrics and divergences beyond the Wasserstein distance that are 
more closely related to implementable generative models.

\section*{Acknowledgments}
This work was  supported by the National Science Foundation grants 
NSF-DMS-2208535 (Machine Learning for Bayesian Inverse Problems) and 
NSF-DMS-2337678 (CAREER: Gaussian Processes for Scientific Machine
Learning: Theoretical Analysis and Computational Algorithms).

\bibliographystyle{siamplain}
\bibliography{references, references-divergence-paper}

\appendix 

\section{Sequential training of WGANs}\label{app:WGAN-training}
In the 2D benchmark experiment, we trained a WGAN-gp to generate the priors. In order to increase the accuracy of training result and illustrate the effect of sample size, network width, and training epochs, we applied a sequential training method to our WGAN-gp. 
\subsection{Default Setting for Sequential Training of WGAN-gp} \label{sec:default setting}
Our default WGAN-gp which produced \Cref{Fig_2DResult} was set up with following details:
\begin{itemize}
\item {\it Basic structure of a single model:} Our generator and discriminator in a single model were set to be 3 layers and each layer has width of 128. The input and output dimension of our generator were set to be 2D. After testing our training model, we noticed the discriminator training loss would be unstable with 2D input. To make the training process stable, we add a zero vectors in the input of discriminator so that the input of discriminator is 3D with output dimension to be 1D.

\item{\it Hyper-parameters for training: }In the training process, our default setting is 10000 data points with batch size of 512, the number of total epochs is 10000. During the training, we update our generator every 20 epochs to make sure the discriminator is fully trained. Lastly, our gradient penalty weight was set to be 5 and the learning rate was 0.01. The optimizer we used was SGD.

\item{\it Setting up for sequential training: } After setting up the single WGAN-gp, we use a loop to sequentially train 5 models $G_1, \dots, G_5$. In the first step, the input data is $x_1 \sim N(0,I)$. To keep updates small and smooth, we set the output to be residual such that $y_1 = x_1+G(x_1)$. Then, starts from $G_2$, we choose the input $x_i = y_i = x_{i-1}+G(x_{i-1})$ to sequentially train $G_2, \dots ,G_5$. With these sequential models, our generative priors were accurate enough to illustrate the effect of sample size, network width, and training epochs.
\end{itemize}

\subsection{Hyper-parameter Tuning for analyzing of effect of sample size, model width and training epochs}
\begin{itemize}
\item {\it Tuning for analyzing effect of training sample: } In the experiment of analyzing effect of training sample, we keep our WGAN-gp's architecture with 3 layers and with width of 128. The sequential training models were kept as 5 as well. The training epochs were set to be 10000 and all other hyper-parameters were set as in \Cref{sec:default setting}. With those hyperparameters to be fixed, we change the number of training samples from $2^9$ to $2^13$ for the distribution of swissroll and pinwheel, and from $2^10$ to $2^14$ for the distribution of checkerboard to obtain \Cref{WDsample complexity}.
\item{\it Tuning for analyzing effect of network width: } In the experiment of analyzing effect of width for our WGAn-gp, we keep the training epochs to be 10000 and the number of training samples to be 10000. We also kept the number of layers for each neural network as 3 and the sequential models to be 5. With all other hyper-parameters fixed as in \Cref{sec:default setting}, we modified the width of the neural network from $2^4$ to $2^8$ to obtain \Cref{WDwidth}.
\item{\it Tuning for analyzing effect of training epochs: } In the experiment of analyzing effect of training epochs, we keep the architecture of our networks as 3 layers with width of 128 and 5 steps for sequential training models. The training samples were set to be 10000 and all other hyperparameters were set as in \Cref{sec:default setting}. Then, we changed our training epochs from $2^8$ to $2^12$ to get \Cref{WDepochs}.
\end{itemize}
\section{WGANgp Training for High Dimension Experiments}\label{app: HDWGANgp}
In the High Dimension Experiments, we trained a WGAN-gp to generate MNIST images. Here is the detail:
\begin{itemize}
\item{\it Architecture of the network: } The generator of our WGAN-gp for high-dimensional problem has the input of dimension 64. It then has three internal layers of dimension 256, 512, and 1024. The output dimension needs to be agree with the size of MNIST image, which was 784 (28 by 28). The discriminator has the input dimension of 784 as it needs to be agree with the MNIST image dimension. It has two layers of 512 and 256, with the output dimension as 1. For both the generator and discriminator, inbetween each layers, we set the LeakyReLU activation function with the parameter to be 0.2.

\item{\it Hyperparameters: } To train the WGAN-gp for generating MNIST image, we used the whole MNIST dataset and the batch size was 64. Our total number of epochs was 600 and we update the parameters of generator every 5 epochs. Our learning rate was $10^{-4}$ and the gradient penalty weight was set to be 5. Lastly, we choose Adam as our optimizer for this training task.
\end{itemize}

\section{Finite Element method for Darcyflow PDE solver}\label{app: FEM} To solve the Darcyflow PDE numerically, our Finite Element method was defined as follows: Since we were given the generated MNIST image as input in the size of 28 by 28, we stick with the same size for our preliminary fields in the domain $(0,1)^2$. We applied the Dirichlet Boundary conditions in the domain and the output numerical results were still in the size of 28 by 28.

\section{MCMC for sampling Darcyflow PDE solution}\label{app: MCMC}
The MCMC algorithm for sampling the posterior were defined in \Cref{DarcyAlg}. For each sampling task, we set the first 20\% steps to be the burning up process in order to tuning the step size parameter $\beta$. During the burning up process, for every 100 steps, if the average acceptance rate $\alpha_{avg} > 0.4$, we set $\beta_{new} = \beta /2$; if the average acceptance rate $\alpha_{avg} < 0.2$, we set $\beta_{new} = ((1-beta)/2)$. Therefore, with enough burning up steps, we could control the acceptance rate after burning up to be around $0.3$, which was usually a proper acceptance rate for the effectiveness of MCMC algorithm.
\end{document}